\documentclass{article}

\usepackage{microtype}
\usepackage{graphicx}
\usepackage{booktabs} 
\usepackage{caption}
\usepackage{subcaption}
\usepackage{wrapfig}
\usepackage{hyperref}
\usepackage{enumitem}

\usepackage{amsmath,amsfonts,bm}

 \newenvironment{myitemize}{\begin{list}{$\bullet$}
{\setlength{\topsep}{1mm}
\setlength{\itemsep}{0.25mm}
\setlength{\parsep}{0.25mm}
\setlength{\itemindent}{0mm}
\setlength{\partopsep}{0mm}
\setlength{\labelwidth}{15mm}
\setlength{\leftmargin}{4mm}}}{\end{list}}

\def\eqref#1{equation~\ref{#1}}

\def\1{\bm{1}}

\DeclareMathAlphabet{\mathsfit}{\encodingdefault}{\sfdefault}{m}{sl}
\SetMathAlphabet{\mathsfit}{bold}{\encodingdefault}{\sfdefault}{bx}{n}

\newcommand{\li}[1]{{{\color{magenta}\textbf{}}}}

\usepackage[arxiv]{icml2024}

\usepackage{amsmath}
\usepackage{amssymb}
\usepackage{mathtools}
\usepackage{amsthm}
\usepackage{makecell}
\usepackage[capitalize,noabbrev]{cleveref}

\theoremstyle{plain}
\newtheorem{theorem}{Theorem}
\newtheorem{proposition}{Proposition}
\newtheorem{lemma}{Lemma}
\newtheorem{corollary}{Corollary} 
\theoremstyle{definition}
\newtheorem{definition}{Definition}
\newtheorem{assumption}{Assumption}
\theoremstyle{remark}
\newtheorem{remark}{Remark}
\newtheorem{example}{Example}
\usepackage[textsize=tiny]{todonotes}

\icmltitlerunning{PAGAR: Taming Reward Misalignment in Inverse Reinforcement Learning-Based Imitation Learning with \underline{P}rotagonist \underline{A}ntagonist \underline{G}uided \underline{A}dversarial \underline{R}eward}

\begin{document}

\twocolumn[
\icmltitle{PAGAR: Taming Reward Misalignment in \\ 
Inverse Reinforcement Learning-Based Imitation Learning with \\ \underline{P}rotagonist \underline{A}ntagonist \underline{G}uided \underline{A}dversarial \underline{R}eward}



\icmlsetsymbol{equal}{*}

\begin{icmlauthorlist}
\icmlauthor{Weichao Zhou}{yyy}
\icmlauthor{Wenchao Li}{yyy}
\end{icmlauthorlist}

\icmlaffiliation{yyy}{Department of Electrical and Computer Engineering, Boston University, Boston, United States}

\icmlcorrespondingauthor{Weichao Zhou}{zwc662@bu.edu}

\icmlkeywords{Machine Learning, ICML}

\vskip 0.3in
]



\printAffiliationsAndNotice{}  

\begin{abstract}
Many imitation learning (IL) algorithms employ inverse reinforcement learning (IRL) to infer the intrinsic reward function that an expert is implicitly optimizing for, based on their demonstrated behaviors. 
However, in practice, IRL-based IL can fail to accomplish the underlying task due to a misalignment between the inferred reward and the objective of the task.
In this paper, we address the susceptibility of IL to such misalignment by introducing a semi-supervised reward design paradigm called Protagonist Antagonist Guided Adversarial Reward (PAGAR).
PAGAR-based IL trains a policy to perform well under mixed reward functions instead of a single reward function as in IRL-based IL.
We identify the theoretical conditions under which PAGAR-based IL can avoid the task failures caused by reward misalignment.
We also present a practical on-and-off policy approach to implementing PAGAR-based IL. 
Experimental results show that our algorithm outperforms standard IL baselines in complex tasks and challenging transfer settings.
\end{abstract}

\section{Introduction}\label{sec:intro}
The central principle of reinforcement learning (RL) is reward maximization~\cite{mnih2015,silver2016,bertsekas2009}. 
The effectiveness of RL thus hinges on having a proper reward function that drives the desired behaviors~\cite{silver2021}.
\li{Suggest to change to something like ``Effectiveness of RL hinges on having a proper reward ... ''} 
Inverse reinforcement learning (IRL)~\cite{ng2000, finn2017} is a well-known approach that aims to learn an expert's reward by observing the expert demonstrations.
IRL is also often leveraged as a subroutine in imitation learning (IL) where the learned reward function is used to train a policy via RL~\cite{abbeel2004,gail}. 
However, the reward function inferred via IRL can be misaligned with the true task objective. 
One common cause of the misalignment is reward ambiguity -- multiple reward functions can be consistent with expert demonstrations even when there are infinite data~\cite{ng2000,cao2021,skalse2022,skalse2022gaming}. 
Another crucial cause is false assumptions about how the preferences of the expert relate to the expert's demonstrations~\cite{skalse2023}.
Training a policy with a misaligned reward can result in reward hacking and task failures~\cite{ird,amodei2016,pan2022}.

In this paper, we consider a reward function as aligning with a task, which is unknown to the learner, if the performances of policies under this reward function accurately indicate whether the policies succeed or fail in accomplishing the task.
Especially, we consider tasks that can be specified by a binary mapping $\Phi:\Pi\rightarrow \{\texttt{true}, \texttt{false}\}$, where $\Pi$ is the policy space, $\Phi(\pi) = \texttt{true}$ indicates that a policy $\pi\in\Pi$ succeeds in the task, and $\Phi(\pi) = \texttt{false}$ indicates that $\pi$ fails.
An example of such a task could be ``reach a target state with a probability greater than $90\%$".
Given any reward function $r$, we use the utility function $U_r(\pi)=\mathbb{E}_{\tau\sim\pi}[r(\tau)]$ to measure the performance of $\pi$ under $r$, and use $\mathbb{U}_r=[\underset{\pi\in\Pi}{\min}\ U_r(\pi), \underset{\pi\in\Pi}{\max}\ U_r(\pi)]$ to represent the range of $U_r$ for $\pi\in \Pi$.
Then, we formally define task-reward alignment in Definition \ref{def:sec1_1}. 

\begin{definition}[\textbf{Task-Reward Alignment}]\label{def:sec1_1}
Assume that there is an unknown task specified by $\Phi$. 
A reward function $r$ is \textbf{aligned} with this task if and only if there exist two intervals $S_r, F_r\subset \mathbb{U}_r$ that satisfy the following two conditions: 
(1) $(S_r\neq \emptyset) \wedge (\sup S_r = \underset{\pi\in\Pi}{\max}\ U_r(\pi))\wedge (F_r\neq \emptyset \Rightarrow \inf S_r>\sup F_r)$,
(2) for any policy $\pi\in \Pi$, $U_{r}(\pi)\in S_r \Rightarrow \Phi(\pi)$ and $U_{r}(\pi)\in F_r\Rightarrow \neg \Phi(\pi)$.
Otherwise, $r$ is considered \textbf{misaligned}.
\end{definition}

The definition suggests that when a reward function $r$ is aligned with the task, it is guaranteed that any policy achieving a high enough utility under $r$ can accomplish the task, and any policy achieving a low utility fails the task.
As for the policies that achieve utilities between the $S$(success) and $F$(failure) intervals, there is no guarantee for their success or failure.
When the reward function is misaligned with the task, even the optimal policy under this reward function cannot accomplish the task.
We visualize Definition \ref{def:sec1_1} in Figure \ref{fig:intro_1} where  $r^+$ is aligned with the underlying task while $r^-$ is misaligned.
The optimal policy $\pi^-$ under $r^-$ fails the task since its utility $U_{r^+}(\pi^-)$ under $r^+$ is in the $F_{r^+}$ interval.
The optimal policy $\pi^+$ under $r^+$ can accomplish the task since its utility is the supremum of $S_{r^+}$.
As the task is unknown, there is no guarantee that the reward function learned via IRL aligns with the task, exposing IRL-based IL to potential task failure.
Our idea to tackle this challenge is to retrieve a set of reward functions under which the expert demonstrations perform well and then train the agent policy also to achieve high utilities under all those reward functions.
The rationale is to use expert demonstrations as weak supervision for locating the task-aligned reward functions.
From the perspective of imitation, we can gain from those reward functions a collective validation of the similarity between the agent policy and the expert.
For instance, in Figure \ref{fig:intro_1}, IRL may infer $r^-$ as the optimal reward function and $r^+$ as a sub-optimal reward function. 
IRL-based IL will only use $r^-$ to learn its optimal policy $\pi^-$, thus failing the task.
Our aim is to learn $\pi^+$, which can achieve high utilities under both $r^+$ and $r^-$, just like the expert demonstration $E$.

\begin{figure}[!t]
    \includegraphics[width=0.48\textwidth]{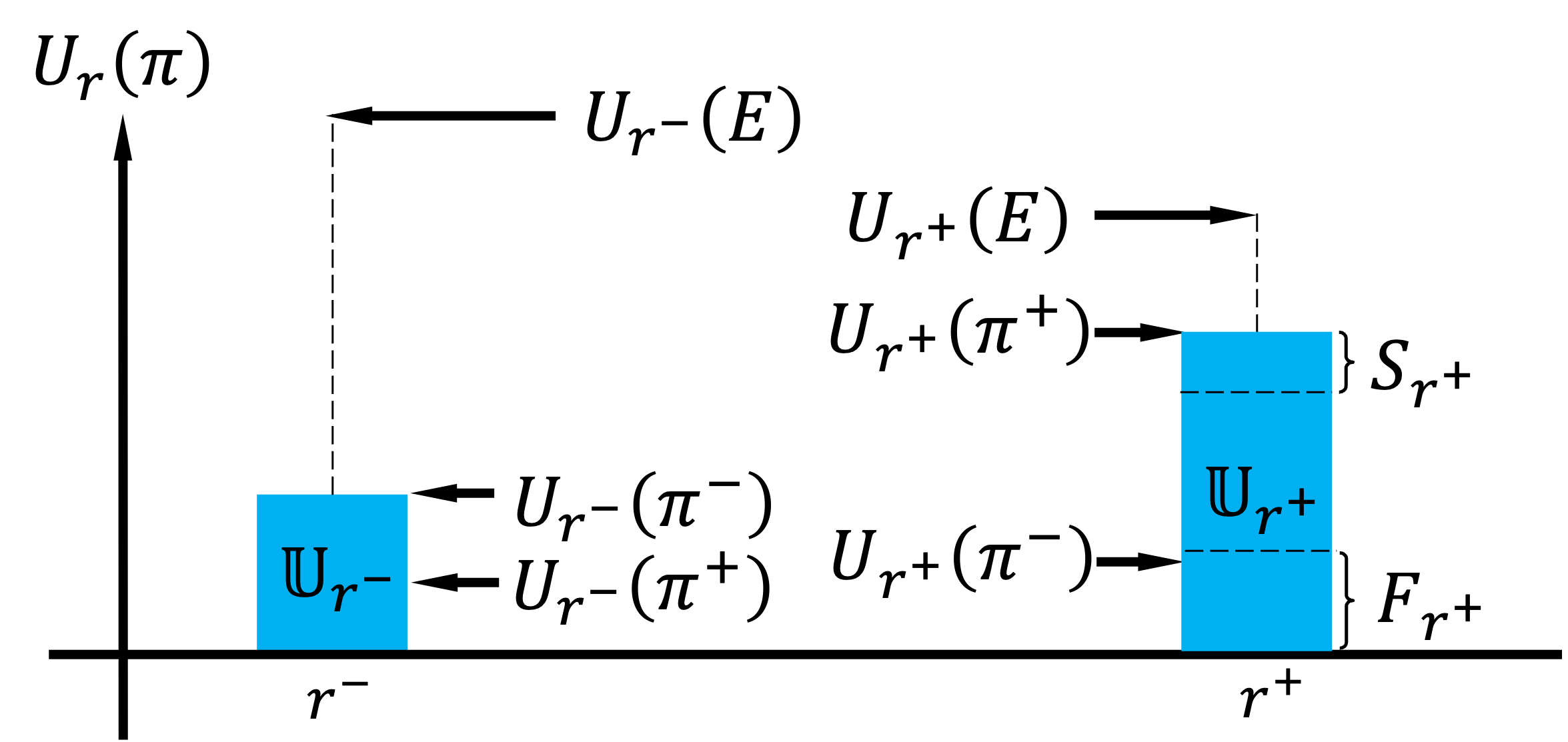}%
  \caption{
  Between the two reward functions, $r^-$ is misaligned with the task, and $r^+$ is aligned with the task.
  The vertical axis measures the ranges of $U_{r^+}(\pi)$ and $U_{r^-}(\pi)$ for $\forall \pi\in\Pi$. 
  $\mathbb{U}_{r^+}$ is the interval $[\underset{\pi\in\Pi}{\min}\ U_{r^+}(\pi), \underset{\pi\in\Pi}{\max}\ U_{r^+}(\pi)]$ where
  $S_{r^+}$ and $F_{r^+}$ are two disjoint intervals such that any policy achieving higher utility than $\inf S_{r^+}$ can succeed in the task, and that achieving lower utility than $\sup F_{r^+}$ fails.
  $\mathbb{U}_{r^-}$ is the interval $[\underset{\pi\in\Pi}{\min}\ U_{r^-}(\pi), \underset{\pi\in\Pi}{\max}\ U_{r^-}(\pi)]$ where the $S$ and $F$ intervals do not exist.
  The reward function $r^-$ is the optimal solution for IRL since $E$ maximally outperforms any other policy under $r^-$.
  IRL-based IL learns its optimal policy $\pi^-$, which, however, has a low utility $U_{r^+}(\pi^-)\leq \sup F_{r^+}$ under $r^+$, thus failing the task.
  In contrast, $\pi^+$ performs consistently well under $r^+$ and $r^-$ as $E$ does.
  }\label{fig:intro_1}
  \vspace{-3mm}
\end{figure}

We concretize our proposition in a novel semi-supervised reward design paradigm called {\underline{P}rotagonist \underline{A}ntagonist \underline{G}uided \underline{A}dversarial \underline{R}eward (PAGAR)}.
Treating the policy trained for the underlying task as a protagonist policy, PAGAR adversarially searches for reward functions to challenge the protagonist policy to be on par with the optimal policy under that reward function (which we call an antagonist policy).
We only consider the reward functions under which the expert demonstrations achieve high utility and iteratively train the protagonist policy with those reward functions. 
Experimental results show that our algorithm outperforms baselines on complex IL tasks with limited demonstrations.
We summarize our contributions below.
\begin{myitemize}
\item We propose PAGAR, a semi-supervised reward design paradigm for mitigating reward misalignment. 
\item We identify the technical conditions for PAGAR-based IL to avoid failures in the underlying task.
\item We develop an on-and-off-policy approach to implementing this paradigm in IL.  
\end{myitemize}

\section{Related Works}\label{sec:relatedworks}

IRL-based IL circumvents many challenges of traditional IL such as compounding error \cite{ross2010,ross2011,zhou2020runtime} by learning a reward function to interpret the expert behaviors \cite{ng1999,ng2000} and then learning a policy from the reward function via reinforcement learning (RL)\cite{sutton2018rl}. 
However, the learned reward function may not always align with the underlying task, leading to reward misspecification~\cite{pan2022,skalse2023}, reward hacking~\cite{skalse2022gaming}, and reward ambiguity~\cite{ng2000,cao2021}.
The efforts on alleviating reward ambiguity include Max-Entropy IRL  \cite{maxentirl}, Max-Margin IRL  \cite{abbeel2004,ratliff2006}, and Bayesian IRL  \cite{bayesirl2007}. 
GAN-based methods~\cite{gail,bayesgail2018,finn2016,vail,airl} use neural networks to learn reward functions from limited demonstrations.
However, those efforts target reward ambiguity but fall short of mitigating the general impact of reward misalignment, which can be caused by various reasons such as IRL making false assumptions about the human model \cite{skalse2022, hong2022sensitivity}.
Our approach seeks to mitigate the general impact of reward misalignment.
Other attempts to mitigate reward misalignment involve external information other than expert demonstrations~\cite{hejna2023inverse, cav2018,aaai2022,icml2022}.
Our work adopts the generic setting of IRL-based IL without involving additional information.
The idea of considering a reward set instead of focusing on a single reward function is supported by \cite{metelli2021} and \cite{lindner2022}.
However, we target reward ambiguity, not reward misalignment.
Our protagonist and antagonist setup is inspired by the concept of unsupervised environment design (UED) from \cite{paired}.
In this paper, we develop novel theories in the context of reward learning.
\section{Preliminaries}\label{sec:prelim}
\noindent \textbf{Reinforcement Learning (RL)} models the environment as a Markov Decision process $\mathcal{M}=\langle \mathbb{S, A}, \mathcal{P}, d_0\rangle$ where $\mathbb{S}$ is the state space,
$\mathbb{A}$ is an action space, $\mathcal{P}(s'|s, a)$ is the probability of reaching a state $s'$ by performing an action $a$ at a state $s$, and $d_0$ is an initial state distribution. 
A \textit{policy} $\pi(a|s)$ determines the probability of an RL agent performing an action $a$ at state $s$. 
By successively performing actions for $T$ steps from an initial state $s^{(0)}\sim d_0$, a \textit{trajectory} $\tau=s^{(0)}a^{(0)}s^{(1)}a^{(1)}\ldots s^{(T)}$\li{including $a^{(T)}$?} is produced.
A state-action based \textit{reward function} is a mapping $r:S\times A\rightarrow \mathbb{R}$. 
The soft Q-value function of $\pi$ is $\mathcal{Q}_\pi(s, a)=r(s, a) + \gamma \cdot \underset{s'\sim \mathcal{P}(\cdot|s, a)}{\mathbb{E}}\left[\mathcal{V}_\pi(s')\right]$ where $\mathcal{V}_\pi$ is the soft state-value function of $\pi$ defined as $\mathcal{V}_\pi(s):=\underset{a\sim \pi(\cdot|s)}{\mathbb{E}}\left[\mathcal{Q}_\pi(s, a)\right] + \mathcal{H}(\pi(\cdot|s))$, and $\mathcal{H}(\pi(\cdot|s))$ is the entropy of $\pi$ at a state $s$. 
The soft advantage of performing action $a$ at state $s$ then following a policy $\pi$ afterwards is then $\mathcal{A}_\pi(s,a)=\mathcal{Q}_\pi(s, a)-\mathcal{V}_\pi(s)$. 
The expected return of $\pi$ under a reward function $r$ is given as $U_r(\pi)=\underset{\tau\sim \pi}{\mathbb{E}}[\sum^T_{t=0} r(s^{(t)}, a^{(t)})]$.
\li{the $U$ notation is uncommon; not a major issue though}
With a little abuse of notations, we denote $r(\tau):=\sum^T_{t=0} \gamma^t \cdot r(s^{(t)}, a^{(t)})$, and $\mathcal{H}(\pi):= \sum^T_{t=0}\underset{s^{(t)}\sim \pi}{\mathbb{E}}[\gamma^t \cdot\mathcal{H}(\pi(\cdot|s^{(t)}))]$.
The standard RL learns a policy by maximizing $U_r(\pi)$.
The entropy regularized RL learns a policy by maximizing the objective function $J_{RL}(\pi;r):=U_r(\pi)+\mathcal{H}(\pi)$.

\noindent\textbf{Inverse Reinforcement Learning (IRL)} assumes that a set $E = \{\tau_1,\ldots, \tau_N\}$  of expert demonstrations are sampled from the roll-outs of the expert's policy $\pi_E$ which is optimal under an expert reward $r_E$.
IRL~\cite{ng2000} learns $r_E$ by solving for the reward function $r$ that maximizes the margin $U_r(E) - \underset{\pi}{\max}\  U_r(\pi) $. 
Maximum Entropy IRL (MaxEnt IRL)~\cite{maxentirl} further proposes an entropy regularized objective function $J_{IRL}(r)=U_r(E) - (\underset{\pi}{\max}\  U_r(\pi) + \mathcal{H}(\pi))$.
The generic IRL-based IL learns $\pi_E$ via  $\arg\underset{\pi\in \Pi}{\max}\ J_{RL}(\pi; r^*)\  s.t.\ r^*=\arg\underset{r}{\max}\ J_{IRL}(r)$.

\noindent \textbf{Generative Adversarial Imitation Learning (GAIL)}~\cite{gail} draws a connection between IRL and Generative Adversarial Nets (GANs) as shown in Eq.\ref{eq:prelm_1} where a discriminator $D:\mathbb{S}\times\mathbb{A}\rightarrow [0,1]$ is trained by minimizing Eq.\ref{eq:prelm_1} so that $D$ can accurately identify any $(s,a)$ generated by the agent. Meanwhile, an agent policy $\pi$ is trained as a generator to maximize Eq.\ref{eq:prelm_1} so that $D$ cannot discriminate $\tau\sim \pi$ from $\tau_E$. 
Adversarial inverse reinforcement learning (AIRL)~\cite{airl} uses a neural network reward function $r$ to represent $D(s,a):=\frac{\pi(a|s)}{\exp(r(s,a)) + \pi(a|s)}$, rewrites $J_{IRL}$ as minimizing Eq.\ref{eq:prelm_1}, and proves that the optimal $r$ satisfies $r^*\equiv \log \pi_E\equiv \mathcal{A}_{\pi_E}$. 
By training $\pi$ with $r^*$ until optimality, $\pi$ will behave just like $\pi_E$. 
\begin{eqnarray}
\underset{{(s, a)\sim \pi}}{\mathbb{E}}\left[\log D(s, a))\right]+ \underset{{(s, a)\sim\pi_E}} {\mathbb{E}}\left[\log  (1 - D(s,a)) \right]\label{eq:prelm_1}
\end{eqnarray}

\section{Protagonist Antagonist Guided Adversarial Reward (PAGAR)}\label{sec:pagar1}

In this section, we formalize the reward misalignment problem in IRL-based IL and
introduce our semi-supervised reward design paradigm, PAGAR.
We then theoretically analyze how PAGAR-based IL can mitigate reward misalignment.

\subsection{Reward Misalignment in IRL-Based IL}\label{subsec:pagar1_1}

In IRL-based IL, we use $\pi_{r^*}$ to denote the optimal policy under the optimal reward function $r^*$ learned through IRL.
According to Definition \ref{def:sec1_1}, we can derive Lemma \ref{lm:sec4_1}, of which the proof is in Appendix \ref{subsec:app_a_5}. 
\begin{lemma} \label{lm:sec4_1}
    The optimal solution $r^*$ of IRL is misaligned with the task specified by $\Phi$ iff $\Phi(\pi_{r^*})\equiv false$.
\end{lemma}
Since the underlying task is unknown to the IRL agent, there is no guarantee that the learned reward function $r^*$ aligns with the task.
Lemma \ref{lm:sec4_1} implies that IRL-based IL is susceptible to reward misalignment because it trains an agent policy to be optimal only under $r^*$.
To overcome this challenge, we leverage expert demonstrations as weak supervision by considering a set of reward functions under which the expert demonstrations perform well, rather than relying solely on the single optimal reward function of IRL.
Specifically, we introduce a hyperparameter $\delta$ and define a \textit{$\delta$-optimal reward function set} $R_{E,\delta}:=\{r| J_{IRL}(r) \geq \delta\}$ that encompasses the reward functions under which the expert demonstrations outperform optimal policies by at least $\delta$.
This $\delta$ is upper-bounded by  $\delta^*:=\underset{r}{\max}\ J_{IRL}(r)$.
If $\delta =\delta^*$, $R_{E,\delta^*}$ only contains the optimal reward function(s) of IRL.
Otherwise, $R_{E,\delta}$ contains both the optimal and sub-optimal reward functions.
If $\delta$ is selected properly such that there exists a task-aligned reward function $r^+\in R_{E,\delta}$, 
we can mitigate reward misalignment by searching for a policy $\pi$ so that $U_{r^+}(\pi)\in S_{r^+}$.
Hence, we define the mitigation of reward misalignment as a policy search problem.
\begin{definition}[Mitigate Reward Misalignment]\label{def:sec4_2}
To mitigate reward misalignment in IRL-based IL is to learn a $\pi\in \Pi$ such that $U_{r^+}(\pi)\in S_{r^+}$ for some task-aligned reward function $r^+\in R_{E,\delta}$.
\end{definition}
This definition is straightforward since our ultimate goal is to accomplish the task.
However, the major obstacle is that it is unknown which $r\in R_{E,\delta}$ aligns with the task.
To tackle this, we refer back to the definition of $S$ in Definition \ref{def:sec1_1} and the illustration in Figure \ref{fig:intro_1}, which indicate that $U_{r^+}(\pi)\in S_{r^+}$ if and only if $\underset{\pi'\in\Pi}{\max}\ U_{r^+}(\pi')-U_{r^+}(\pi)<\sup S_{r^+} - \inf S_{r^+}$.
This inequality leads to the following strategy: learn a policy to make the left hand side of the inequality, which is the difference between the performance of the policy and the optimal performance under the reward function, as small as possible for all reward functions in $R_{E,\delta}$.
In this case, we refer to this policy as the \textit{protagonist policy} $\pi_P$.
For each reward function $r$, we define the difference between $U_r(\pi_P)$ and the maximum $U_r$ as the regret of $\pi_P$ under $r$, dubbed \textit{Protagonist Antagonist Induced Regret} as in Eq.\ref{eq:pagar1_1}.
The policy $\pi_A$ in Eq.\ref{eq:pagar1_1} is called an \textit{antagonist policy} since it magnifies the regret of $\pi_P$.
We then use this Eq.\ref{eq:pagar1_1} to define our semi-supervised reward design paradigm in Definition~\ref{def:sec4_1}.
\begin{eqnarray}
Regret(\pi_P, r):= \left\{\underset{\pi_A\in\Pi}{\max}\ U_r(\pi_A)\right\} - U_r(\pi_P)\label{eq:pagar1_1} 
\end{eqnarray}
\begin{definition}[Protagonist Antagonist Guided Adversarial Reward \textbf{(PAGAR)}]\label{def:sec4_1}
Given a candidate reward function set $R$ and a protagonist policy $\pi_P$, PAGAR searches for a reward function $r$ within $R$ to maximize the \textit{Protagonist Antagonist Induced Regret}, i.e., $\underset{r\in R}{\max}\ Regret(r, \pi_P)$.
\end{definition}
\textbf{PAGAR-based IL} \textit{learns a policy from $R_{E,\delta}$ by solving objective function $MinimaxRegret(R_{E,\delta})$ defined in Eq.\ref{eq:pagar1_3}.}
\begin{equation}
MinimaxRegret(R):=\arg\underset{\pi_P\in \Pi}{\min}\ \underset{r\in R}{\max}\ Regret(\pi_P, r) \label{eq:pagar1_3}
\end{equation}

\begin{example}\label{exp_1}
Figure \ref{fig:exp_fig0} shows an illustrative example of how PAGAR-based IL mitigates reward misalignment in IRL-based IL.
The task requires that \textit{a policy must visit $s_2$ and $s_6$ with probabilities no less than $0.5$ within $5$ steps}, i.e. $Prob(s_2|\pi)\geq 0.5 \wedge Prob(s_6|\pi)\geq 0.5$ where $Prob(s|\pi)$ is the probability of $\pi$ generating a trajectory that contains $s$ within the first $5$ steps.
It can be derived analytically that a successful policy must choose $a_2$ at $s_0$ with a probability within $[\frac{1}{2},\frac{125}{178}]$.
The reward function hypothesis space is $\{r_\omega|r_\omega(s,a)=\omega \cdot r_1(s,a)+(1-\omega)\cdot r_2(s,a)\}$ where $\omega\in[0, 1]$ is a parameter, $r_1, r_2$ are two features.
Specifically, $r_1(s,a)$ equals $1$ if $s=s_2$ and equals $0$ otherwise,  and $r_2(s,a)$ equals $1$ if $s=s_6$ and equals $0$ otherwise.
Given the demonstrations and the MDP, the maximum negative MaxEnt IRL loss $\delta^*\approx 2.8$ corresponds to the optimal parameter 
 $\omega^*=1$.
 The optimal policy under $r_{\omega^*}$ chooses $a_2$ at $s_0$ with probability $1$ and reaches $s_6$ with probability less than $0.25$, thus failing to accomplish the task. 
 In contrast, for any $\delta < 1.1$, the optimal protagonist policy $\pi_P=MinimaxRegret(R_{E,\delta})$ can succeed in the task as indicated by the grey dashed lines.
This example demonstrates that by considering the sub-optimal reward functions, PAGAR-based IL can mitigate reward misalignment.
More details about this example can be found in Appendix \ref{subsec:app_a_7}.
Next, we will delve into the distinctive mechanism that sets PAGAR-based IL apart from IRL-based IL.
\end{example}
\begin{figure}
\hfill
\includegraphics[width=0.38\linewidth]{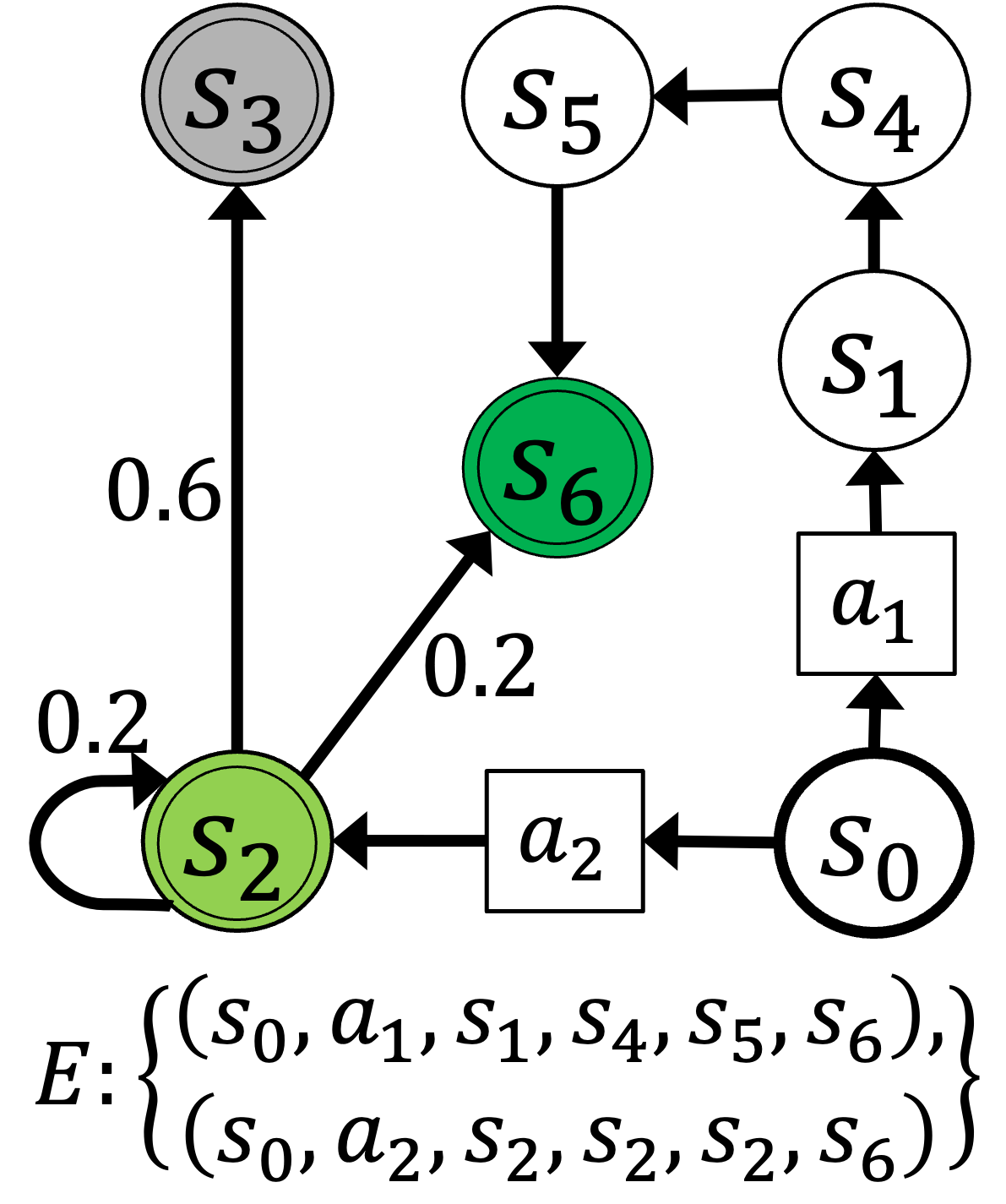}
\hfill
\includegraphics[width=0.6\linewidth]{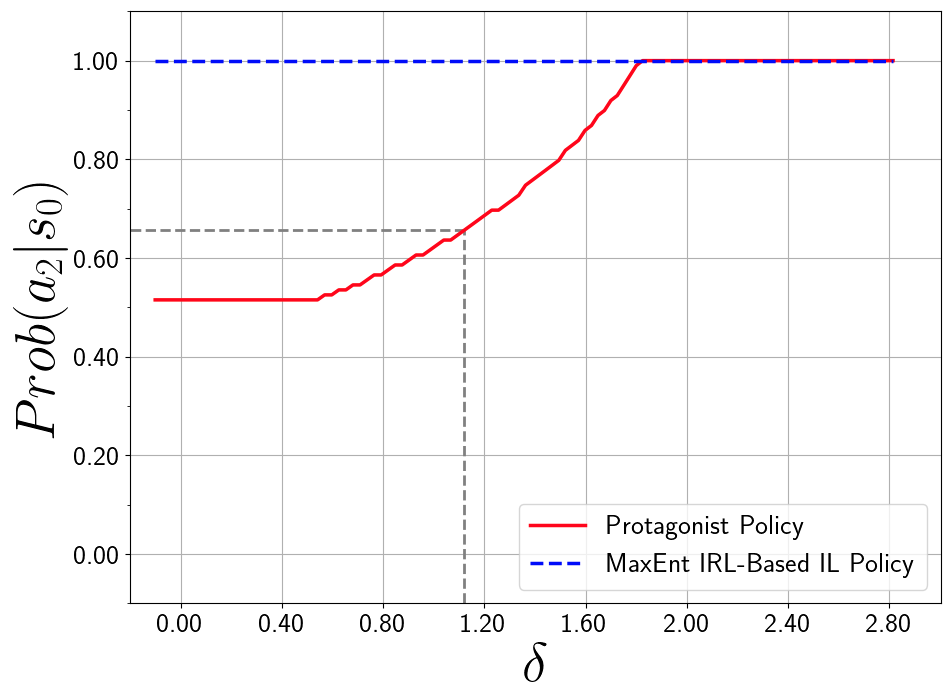}
\hfill
\hfill
\caption{\textbf{Left:} Consider an MDP where there are two available actions $a_1,a_2$ at initial state $s_0$. 
In other states, actions make no difference: the transition probabilities are either annotated at the transition edges or equal $1$ by default. 
States $s_3$ and $s_6$ are terminal states. 
Expert demonstrations are in $E$.
\textbf{Right}: x-axis indicates the MaxEnt IRL loss bound $\delta$ for $R_{E,\delta}$ as defined in Section \ref{subsec:pagar1_1}. 
The y-axis indicates the probability of the protagonist policy learned via $MinimaxRegret(R_{E,\delta})$ choosing $a_2$ at $s_0$.
The red curve shows how different $\delta$'s lead to different protagonist policies.
The blue dashed curve is for reference, showing that the optimal policy under the optimal reward learned of MaxEnt IRL.
} 
\label{fig:exp_fig0}
\end{figure}

\subsection{PAGAR as Mixed Rewarding}
Since $MinimaxRegret$ is formulated as a zero-sum game between $\pi_P$ and $r$ in Eq.\ref{eq:pagar1_3}, we derive an alternative form of $MinimaxRegret$ in Eq.\ref{eq:pagar1_2} which shows that the protagonist policy is trained to be optimal under mixed reward functions.
In a nutshell, Eq.\ref{eq:pagar1_2} searches for a policy $\pi$ with the highest score measured by an affine combination of policy performance $U_r(\pi)$ with $r$ drawn from two different reward function distributions in $R$.
One distribution, $\mathcal{P}_\pi(r)$, is a baseline distribution over $r\in R$ such that: 1) for policies that do not always perform worse than any other policy under $r\in R$, the expected $U_r(\pi)$ values measured under $r\sim \mathcal{P}_\pi$ equals a constant $c$, which is the smallest value for $c\equiv \underset{r\sim \mathcal{P}_\pi}{\mathbb{E}}\ [U_r(\pi)]$, and 2) for any other policy $\pi$, $\mathcal{P}_{\pi}$ uniformly concentrates on $\arg\underset{r\in R}{\max}\ U_r(\pi)$.
The other distribution is a singleton with support on a reward function $r^*_\pi$, which maximizes $ U_r (\pi)$ among all those $r$'s that maximize $Regret(\pi, r)$. 
A detailed derivation, including the proof for the existence of such $\mathcal{P}_\pi$, can be found in Theorem \ref{th:app_a_1} in Appendix \ref{subsec:app_a_2}. 
\begin{eqnarray}
&&MinimaxRegret(R)\nonumber\\
&\equiv& \arg\underset{\pi\in \Pi}{\max}\ \{\frac{Regret(\pi, r^*_\pi)}{c - U_{r^*_\pi}(\pi)}\cdot U_{r^*_\pi}(\pi) +\nonumber\\
&&\qquad \underset{r\sim \mathcal{P}_\pi(r)}{\mathbb{E}}[(1 - \frac{Regret(\pi, r)}{c - U_r(\pi)})\cdot U_r(\pi)]\}\label{eq:pagar1_2}
\end{eqnarray}
In PAGAR-based IL, where $R_{E,\delta}$ is used in place of $R$ in Eq.\ref{eq:pagar1_2}, $\mathcal{P}_\pi$ is a distribution over $R_{E,\delta}$ and $r^*_\pi$ is constrained to be within $R_{E,\delta}$. 
Essentially, the mixed reward functions dynamically assign weights to $r\sim \mathcal{P}_\pi$ and $r^*_\pi$ depending on $\pi$. 
If $\pi$ performs worse under $r^*_\pi$ than under many other reward functions ($U_{r^*_\pi}(\pi)$ falls below $c$), a higher weight will be allocated to using $r^*_\pi$ to train $\pi$.
Conversely, if $\pi$ performs better under $r^*_\pi$ than under many other reward functions ($c$ falls below $U{r^*_\pi}(\pi)$), a higher weight will be allocated to reward functions drawn from $\mathcal{P}_\pi$. 
Furthermore, we prove in Appendix \ref{subsec:app_a_6} that $MinimaxRegret$ is a convex optimization w.r.t the protagonist policy $\pi_P$.

Our subsequent discussion will identify the conditions where PAGAR-based IL can mitigate reward misalignment.

\subsection{Mitigate Reward Misalignment with PAGAR}
We start by analyzing the general properties of $MinimaxRegret$. 
Following the notations in Definition \ref{def:sec1_1}, we denote any task-aligned reward function as $r^+\in R$ and misaligned reward function as $r^-\in R$.
We then derive a sufficient condition for $MinimaxRegret$ to induce a policy that avoids task failure as in Theorem \ref{th:pagar_1_1}.

\begin{theorem}[Task-Failure Avoidance]\label{th:pagar_1_1}
If the following conditions (1) (2) hold for $R$, then the optimal protagonist policy $\pi_P:=MinimaxRegret(R)$ satisfies $r^+\in R$,$ U_{r^+}(\pi_P)\notin F_{r^+}$. \li{seems to be missing a connective before forall}
\begin{enumerate}[topsep=0pt,itemsep=-1ex,partopsep=1ex,parsep=1ex]
\item[(1)] There exists $r^+\in R$, and $\underset{r^+\in R}{\max}\ \{\sup F_{r^+} - \inf F_{r^+}\} < \underset{{r^+}\in R}{\min}\ \{\inf S_{r^+} - \sup F_{r^+}\} \wedge \underset{{r^+}\in R}{\max}\ \{\sup S_{r^+} - \inf S_{r^+}\} < \underset{{r^+}\in R}{\min}\ \{\inf S_{r^+} - \sup F_{r^+}\}$; \\
\item[(2)]  There exists a policy $\pi^*$ such that $\forall r^+\in R$, $U_{r^+}(\pi^*)\in S_{r^+}$, and $\forall r^-\in R$,  $\underset{\pi\in\Pi}{\max}\ U_{r^-}(\pi) - U_{r^-}(\pi^*)  <  \underset{{r^+}\in R}{\min}\ \{\inf S_{r^+} - \sup F_{r^+}\}$.
\end{enumerate}
\end{theorem}
In Theorem \ref{th:pagar_1_1}, condition (1) states that for each $r^+\in R$, the ranges of the utilities of successful and failing policies are distributed in small ranges. Condition (2) states that there exists a $\pi^*$ that not only performs well under all $r^+$'s (thus succeeding in the task), but also achieves high performance\li{performance is measured under?} under all $r^-$'s.
The proof can be found in Appendix \ref{subsec:app_a_3}.
Furthermore, Theorem \ref{th:pagar_1_0} shows that, under a stronger condition on the existence of a policy $\pi^*$ performing well under all reward functions in $R$, $MinimaxRegret(R)$ can guarantee to induce a policy that succeeds in the underlying task.

\begin{theorem}[Task-Success Guarantee\li{what guarantee?}]\label{th:pagar_1_0}
Assume that Condition (1) in Theorem \ref{th:pagar_1_1} is satisfied. 
In addition, if there exists a policy $\pi^*$ such that $\forall r\in R$, $\underset{\pi\in\Pi}{\max}\ U_{r}(\pi) - U_{r}(\pi^*)  <  \underset{{r^+}\in R}{\min}\ \{\sup S_{r^+} - \inf S_{r^+}\}$, then the optimal protagonist policy $\pi_P:=MinimaxRegret(R)$ satisfies $\forall r^+\in R$, $U_{r^+}(\pi_P)\in S_{r^+}$. 
\end{theorem}
 
We note that the conditions in Theorem \ref{th:pagar_1_1} and \ref{th:pagar_1_0} are not trivially satisfiable for arbitrary $R$, e.g., if $R$ contains two reward functions with opposite signs, i.e., $r, -r\in R$, no policy can perform well under both $r$ and $-r$.
However, in PAGAR-based IL, using $R_{E,\delta}$ in place of arbitrary $R$ is equivalent to using $E$ and $\delta$ to constrain the selection of reward functions such that there exist policies (such as the expert) that perform well under $R_{E,\delta}$.
This constraint leads to additional implications.
In particular, we consider the case of Maximum Margin IRL where $J_{IRL}(r):=U_r(E)-\underset{\pi}{\max}\ U_r(\pi)$.
We use $L_r$ to denote the Lipschitz constant of $r(\tau)$, $W_E$ to denote the smallest Wasserstein $1$-distance $W_1(\pi, E)$ between $\tau\sim \pi$ of any $\pi$ and $\tau\sim E$, i.e., $W_E\triangleq \underset{\pi\in\Pi}{\min}\ W_1(\pi, E)$.
Then, we have Corrolary \ref{th:pagar_1_3}.
\begin{corollary}\label{th:pagar_1_3}
If the following conditions (1) (2) hold for $R_{E,\delta}$, then the optimal protagonist policy $\pi_P:=MinimaxRegret(R_{E,\delta})$ satisfies $\forall r^+\in R_{E,\delta}$, $U_{r^+}(\pi_P)\notin F_{r^+}$. 
\begin{enumerate}[topsep=0pt,itemsep=-1ex,partopsep=1ex,parsep=1ex]
\item[(1)] The condition (1) in Theorem \ref{th:pagar_1_1} holds
\item[(2)] $\forall {r^+}\in R_{E,\delta}$, $L_{r^+}\cdot W_E-\delta \leq \sup S_{r^+} - \inf S_{r^+}$ and $\forall r^-\in R_{E,\delta}$,  $L_{r^-}\cdot W_E-\delta  <  \underset{{r^+}\in R_{E,\delta}}{\min}\ \{\inf S_{r^+} - \sup F_{r^+}\}$. 
\end{enumerate}
\end{corollary}
Corollary \ref{th:pagar_1_3} delivers the same guarantee as that of Theorem \ref{th:pagar_1_1} but differs from Theorem \ref{th:pagar_1_1} in that 
condition (2) implicitly requires that for the policy $\pi^*=\arg\underset{\pi\in\Pi}{\min}\ W_1(\pi, E)$, the performance difference between $E$ and $\pi^*$ is small enough under all $r\in R_{E,\delta}$.
The following corollary further suggests that a large $\delta$ close to $\delta^*$ can help $MinimaxRegret(R_{E,\delta})$  gain a better chance of 
 finding a policy to succeed in the underlying task, 
\begin{corollary}\label{th:pagar_1_4}
Assume that the condition (1) in Theorem \ref{th:pagar_1_1} holds for $R_{E,\delta}$.
If for any $r\in R_{E,\delta}$, $L_r\cdot W_E-\delta \leq \underset{{r^+}\in R_{E,\delta}}{\min}\ \{\sup S_{r^+} - \inf S_{r^+}\}$, then the optimal protagonist policy $\pi_P=MinimaxRegret(R_{E,\delta})$ satisfies \li{not sure if there is a missing connective here}$\forall r^+\in R_{E,\delta}$, $U_{r^+}(\pi_P)\in S_{r^+}$. 
\end{corollary}

\subsection{Comparing PAGAR-Based IL with IRL-Based IL}
Both PAGAR-based IL and IRL-based IL use a minimax paradigm.
The key difference is in how they evaluate the similarity between the policies: IRL-based IL assesses the policies' performance difference under a single reward function, while PAGAR-based IL considers a set of reward functions.
Despite the difference, PAGAR-based IL can induce the same results as IRL-based IL under certain conditions.
\begin{assumption}\label{asp:pagar1_1}
$\underset{r}{\max}\ J_{IRL}(r)$ can reach Nash Equilibrium at an optimal reward function $r^*$ and its optimal policy $\pi_{r^*}$.
\end{assumption} 
We make this assumption only to demonstrate how PAGAR-based IL can prevent performance degradation w.r.t. IRL-based IL, which is preferred when IRL-based IL does not have a reward misalignment issue.
We draw two assertions from this assumption.
The first one considers Maximum Margin IRL-based IL and shows that if using the optimal reward function set $R_{E,\delta^*}$ as input to $MinimaxRegret$, PAGAR-based IL and Maximum Margin IRL-based IL have the same solutions.
\begin{proposition}\label{th:pagar_1_2}
$\pi_{r^*}=MinimiaxRegret(R_{E,\delta^*})$.
\end{proposition}
The proof can be found in Appendix \ref{subsec:app_a_4}.
The second assertion shows that if IRL-based IL can learn a policy to succeed in the task, $MinimaxRegret(R_{E,\delta})$ with  $\delta < \delta^*$ can also learn a policy that succeeds in the task under certain condition.
The proof can be found in Appendix \ref{subsec:app_a_4}. 
This assertion also suggests that the designer should select a $\delta$ smaller than $\delta^*$ while making $\delta^*-\delta$ no greater than the expected size of the successful policy utility interval.
\begin{proposition}\label{th:pagar_1_6}
If $r^*$ aligns with the task and $\delta \geq \delta^* - (\sup S_{r^*} - \inf S_{r^*})$, the optimal protagonist policy $\pi_P=MinimiaxRegret(R_{E,\delta})$ is guaranteed to succeed in the task.
\end{proposition}

\section{An On-and-Off-Policy Approach to PAGAR-based IL}\label{sec:pagar2}

In this section, we introduce a practical approach to \li{need revision}solving $MinimaxRegret(R_{E,\delta})$ based on IRL and RL.
In a nutshell, this approach alternates between policy learning and reward searching. 
We first explain how we optimize $\pi_P$, $\pi_A$; then we derive from Eq.\ref{eq:pagar1_3} two reward improvement bounds for optimizing $r$. 
We will also discuss how to incorporate IRL to enforce the constraint $r\in R_{E,\delta}$ with a given $\delta$.\li{4.4? poor referencing}

\subsection{Policy Optimization with On-and-Off Policy Samples}

In practice, given an intermediate learned reward function $r$, we use RL to train $\pi_P$ and $\pi_A$ to respectively maximize and minimize $U_r(\pi_P)-U_r(\pi_A)$ as required in $MinimaxRegret$ in Eq.\ref{eq:pagar1_3}.
In particular, we show that we can use a combination of off-policy and on-policy learning to optimize $\pi_P$.

\textbf{Off-Policy:} According to~\cite{trpo}, $U_r(\pi_P) - U_r(\pi_A) \geq \underset{s\in\mathbb{S}}{\sum} \rho_{\pi_A}(s)\underset{a\in\mathbb{A}}{\sum} \pi_P(a|s)\hat{\mathcal{A}}_{\pi_A}(s,a) - C\cdot \underset{s}{\max}\ D_{TV}(\pi_A(\cdot|s), \pi_P(\cdot|s))^2$ where $\rho_{\pi_A}(s)=\sum^T_{t=0} \gamma^t Prob(s^{(t)}=s|\pi_A)$ is the discounted visitation frequency of $\pi_A$, $\hat{\mathcal{A}}_{\pi_A}$ is the advantage function without considering the entropy, and $C$ is some constant. 
This inequality allows us to train $\pi_P$ by using the trajectories of $\pi_A$: following the theories in \cite{trpo} and \cite{ppo}, we derive from the r.h.s of the inequality a PPO-style objective function $J_{\pi_A}(\pi_P; r):=\mathbb{E}_{s\sim\pi_A}[\min(\xi(s,a)\cdot \hat{\mathcal{A}}_{\pi_A}(s,a), clip(\xi(s,a),1-\sigma,1+\sigma)\cdot \hat{\mathcal{A}}_{\pi_A}(s,a)]$ where $\sigma$ is a clipping threshold, $\xi(s,a)=\frac{\pi_P(a|s)}{\pi_A(a|s)}$ is an importance sampling rate. 
\textbf{On-Policy:} In the meantime, since $\arg\underset{\pi_P}{\max}\ U(\pi_P)-U_r(\pi_A)\equiv \arg\underset{\pi_P}{\max}\ U(\pi_P)$, we directly optimize $\pi_P$ with an online RL objective function $J_{RL}(\pi_P; r)$ as mentioned in Section \ref{sec:intro}. 
As a result, the objective function for optimizing $\pi_P$ is $\underset{\pi_P\in\Pi}{\max}\ J_{\pi_A}(\pi_P; r) + J_{RL}(\pi_P; r)$.
As for $\pi_A$, we only use the on-policy RL objective function, i.e.,  $\underset{\pi_A\in\Pi}{\max}\ J_{RL}(\pi_A; r)$.
\subsection{Regret Maxmization with On-and-Off Policy Samples} 
Given the intermediate learned protagonist and antagonist policy $\pi_P$ and $\pi_A$, according to $MinimaxRegret$, we need to optimize $r$ to maximize $Regret(r, \pi_P)$.
In practice, we solve $\arg\underset{r}{\max}\ U_r(\pi_A)-U_r(\pi_P)$ by treating $\pi_A$ as the proxy of the optimal policy under $r$.
We estimate $U_r(\pi_A)$ and $U_r(\pi_P)$ from the samples of $\pi_A$ and $\pi_P$. 
Then, we extract two reward improvement bounds as in Theorem \ref{th:pagar2_1} to help optimize $r$.
\begin{theorem}\label{th:pagar2_1}
Suppose policy $\pi_2\in\Pi$ is the optimal solution for $J_{RL}(\pi; r)$. Then , the inequalities Eq.\ref{eq:pagar2_1} and \ref{eq:pagar2_2} hold for any policy $\pi_1\in \Pi$, where $\alpha = \underset{s}{\max}\ D_{TV}(\pi_1(\cdot|s), \pi_2(\cdot|s))$, $\epsilon = \underset{s,a}{\max}\ |\mathcal{A}_{\pi_2}(s,a)|$, and $\Delta\mathcal{A}(s)=\underset{a\sim\pi_1}{\mathbb{E}}\left[\mathcal{A}_{\pi_2}(s,a)\right] - \underset{a\sim\pi_2}{\mathbb{E}}\left[\mathcal{A}_{\pi_2}(s,a)\right]$. 
\vspace{-0.1in}
\begin{eqnarray}
    &&\left|U_r(\pi_1) -U_r(\pi_2) - \sum^\infty_{t=0}\gamma^t\underset{s^{(t)}\sim\pi_1}{\mathbb{E}}\left[\Delta\mathcal{A}(s^{(t)})\right]\right|\leq\nonumber\\
    &&\qquad\qquad\qquad\qquad\qquad\qquad\quad\frac{2\alpha\gamma\epsilon}{(1-\gamma)^2}\label{eq:pagar2_1}\\
   &&\left|U_r(\pi_1)-U_r(\pi_2) - \sum^\infty_{t=0}\gamma^t\underset{s^{(t)}\sim\pi_2}{\mathbb{E}}\left[\Delta\mathcal{A}(s^{(t)})\right]\right|\leq\nonumber\\
   &&\qquad\qquad\qquad\qquad\qquad\qquad\quad\frac{2\alpha\gamma(2\alpha+1)\epsilon}{(1-\gamma)^2} \label{eq:pagar2_2}
\end{eqnarray}
\vspace{-0.1in}
\end{theorem}
By letting $\pi_P$ be $\pi_1$ and $\pi_A$ be $\pi_2$, Theorem \ref{th:pagar2_1} enables us to bound $U_r(\pi_A)-U_r(\pi_P)$ by using either only the samples of $\pi_A$ or only those of $\pi_P$.
Following \cite{airl}, we let $r$ be a proxy of $\mathcal{A}_{\pi_2}$ in Eq.\ref{eq:pagar2_1} and \ref{eq:pagar2_2}.
Then we derive two loss functions $J_{R,1}(r;\pi_P, \pi_A)$ and $J_{R,2}(r; \pi_P, \pi_A)$ for $r$ as shown in Eq.\ref{eq:pagar2_3} and \ref{eq:pagar2_4} where $C_1$ and $C_2$ are constants proportional to the estimated maximum KL divergence between $\pi_A$ and $\pi_P$ (to bound $\alpha$~ \cite{trpo}).
The objective function for $r$ is  then $J_{PAGAR}:=J_{R,1} +   {J_{R,2}}$.
\begin{eqnarray}
    &&J_{R,1}(r;\pi_P, \pi_A) \nonumber\\
    &:=& \underset{\tau\sim \pi_A}{\mathbb{E}}\left[\sum^\infty_{t=0}\gamma^t\left(\xi(s^{(t)},a^{(t)})- 1\right)\cdot r(s^{(t)},a^{(t)})\right] \nonumber\\
    &&\qquad\qquad\qquad\qquad +\quad C_1\cdot \underset{(s,a)\sim \pi_A}{\max}|r(s,a)|\  \label{eq:pagar2_3}\\
    &&J_{R,2}(r; \pi_P, \pi_A)  \nonumber\\
    &:=& \underset{\tau\sim \pi_P}{\mathbb{E}}\left[\sum^\infty_{t=0}\gamma^t\left(1 - \frac{1}{\xi(s^{(t)}, a^{(t)})}\right) \cdot r(s^{(t)},a^{(t)})\right]\nonumber\\
    &&\qquad\qquad\qquad\qquad + \quad C_2\cdot \underset{(s,a)\sim \pi_P}{\max}|r(s,a)|\ \label{eq:pagar2_4}
\end{eqnarray}
\subsection{Algorithm for Solving PAGAR-Based IL}
\begin{algorithm}[tb]
\caption{An On-and-Off-Policy Algorithm for Imitation Learning with PAGAR}
\label{alg:pagar2_1}
\textbf{Input}: Expert demonstration $E$, IRL loss bound $\delta$, initial protagonist policy $\pi_P$, antagonist policy $\pi_A$, reward function $r$, Lagrangian parameter $\lambda\geq 0$, maximum iteration number $N$. \\
\textbf{Output}: $\pi_P$ 
\begin{algorithmic}[1] 
\FOR {iteration $i=0,1,\ldots, N$}
\STATE Sample trajectory sets $\mathbb{D}_A\sim \pi_A$ and $\mathbb{D}_P \sim \pi_P$
\STATE \textbf{Optimize $\pi_A$}: estimate $J_{RL}(\pi_A;r)$ with $\mathbb{D}_A$; update $\pi_A$ to maximize $J_{RL}(\pi_A;r)$  
\STATE \textbf{Optimize $\pi_P$}: estimate $J_{RL}(\pi_P;r)$ with $\mathbb{D}_P$; estimate $J_{\pi_A}(\pi_P; \pi_A, r)$ with $\mathbb{D}_A$; update $\pi_A$ to maximize $J_{RL}(\pi_P; r) + J_{\pi_A}(\pi_P; r)$
\STATE \textbf{Optimize $r$}: estimate $J_{PAGAR}(r; \pi_P, \pi_A)$ with $\mathbb{D}_P$ and $\mathbb{D}_A$; estimate $J_{IRL}(r)$ with $\mathbb{D}_A$ and $E$; update $r$ to minimize $J_{PAGAR}(r; \pi_P, \pi_A) + \lambda \cdot (J_{IRL}(r)+\delta)$; then update $\lambda$ based on $J_{IRL}(r)+\delta$
\ENDFOR
\STATE \textbf{return} $\pi_P$ 
\end{algorithmic}
\end{algorithm}
We enforce the constraint $r\in R_{E,\delta}$ by adding to  $J_{PAGAR}(r; \pi_P, \pi_A)$ a penalty term $\lambda \cdot(\delta+J_{IRL})$, where  $\lambda$ is  a Lagrangian parameter.
Then, the objective function for optimizing $r$ is $\underset{r\in R}{\min}\ J_{PAGAR}(r; \pi_P, \pi_A) + \lambda \cdot(\delta + J_{IRL}(r))$. 
In particular, for $\delta=\delta^*$, i.e., to only consider the optimal reward function of IRL, the objective function for $r$ becomes $\underset{r}{\min}\ J_{PAGAR} - \lambda \cdot J_{IRL}$ with a large constant $\lambda$.
Algorithm \ref{alg:pagar2_1} describes our approach to PAGAR-based IL.
The algorithm iteratively alternates between optimizing the policies and the rewards function.
It first trains $\pi_A$ in line 3.
Then, it employs the on-and-off policy approach to train $\pi_P$ in line 4, where the off-policy objective $J_{\pi_A}$ is estimated based on $\mathbb{D}_A$. 
In line 5, while $J_{PAGAR}$ is estimated based on both $\mathbb{D}_A$ and $\mathbb{D}_P$, the IRL objective is only based on $\mathbb{D}_A$.
Appendix \ref{subsec:app_b_3} details how we incorporate different IRL algorithms and update $\lambda$ based on $J_{IRL}$.
\begin{figure*}[tbp!]
\begin{subfigure}[t]{0.24\linewidth}
        \includegraphics[width=1.12\linewidth]{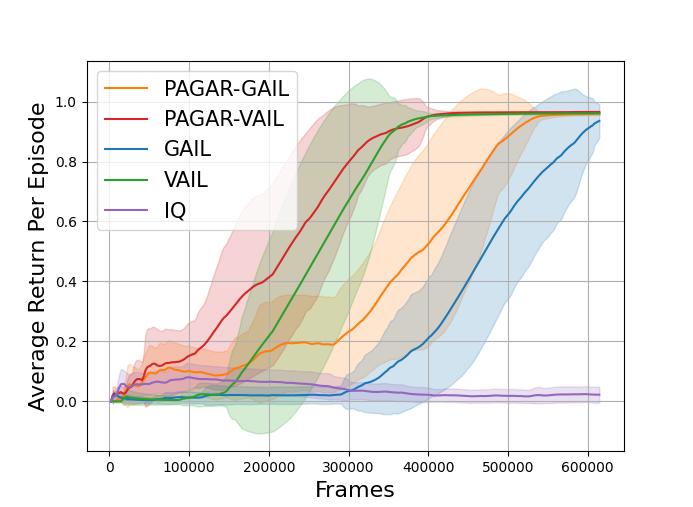}
        \caption{DoorKey-6x6\\ \centering(10 demos)}
    \end{subfigure}    
\hfill
\begin{subfigure}[t]{0.24\linewidth}
        \includegraphics[width=1.1\linewidth]{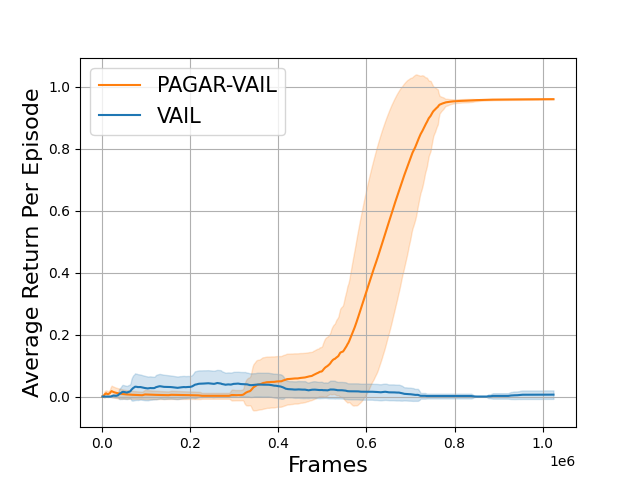}
        \caption{DoorKey-6x6 \\ \centering(1 demo)}
    \end{subfigure}   
\hfill
\begin{subfigure}[t]{0.24\linewidth}
    \includegraphics[width=1.1\linewidth]{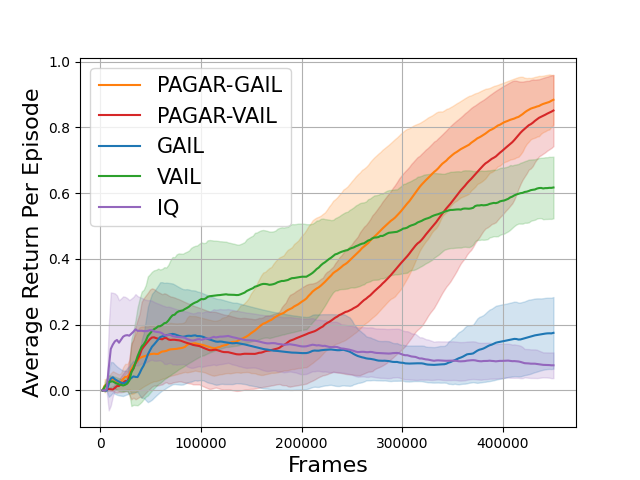}
        \caption{SimpleCrossingS9N1\\\centering (10 demos)}
    \end{subfigure}%
\hfill
\begin{subfigure}[t]{0.24\linewidth}
    \includegraphics[width=1.1\linewidth]{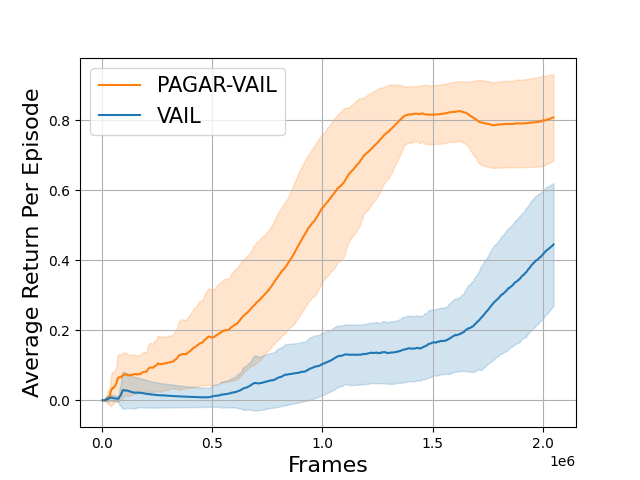}
        \caption{SimpleCrossingS9N1 \\\centering (1 demo)}
    \end{subfigure}%
\\
\phantom{123}
    \begin{subfigure}[b]{0.18\linewidth}
        \includegraphics[width=1\linewidth]{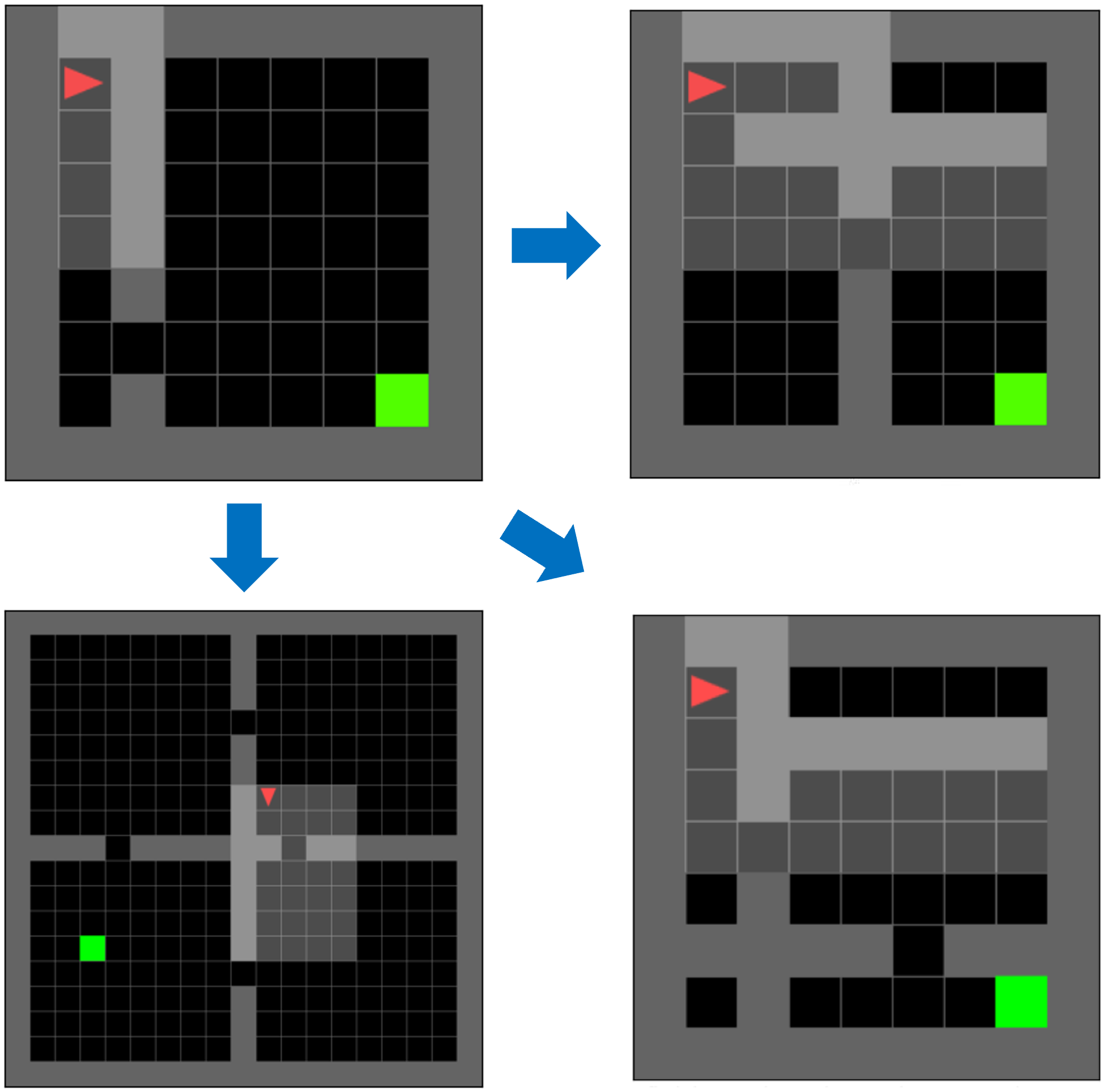}
        \caption{Transfer Envs}
    \end{subfigure}%
\hskip .25in
\begin{subfigure}[b]{0.24\linewidth}
        \includegraphics[width=1.1\linewidth]{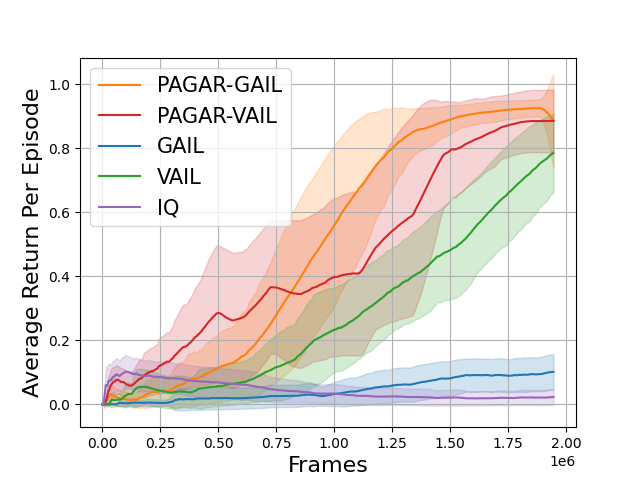}
        \caption{SimpleCrossingS9N2}
    \end{subfigure}%
\hskip .13in
\begin{subfigure}[b]{0.24\linewidth}
        \includegraphics[width=1.1\linewidth]{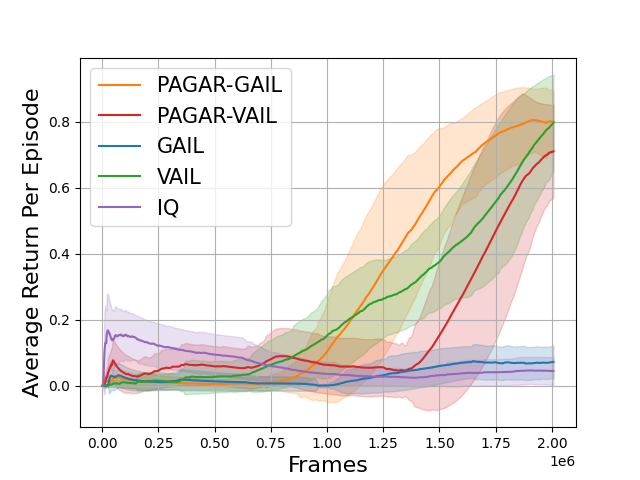}        \caption{SimpleCrossingS9N3}
    \end{subfigure}%
\hfill
\begin{subfigure}[b]{0.24\linewidth}
        \includegraphics[width=1.1\linewidth]{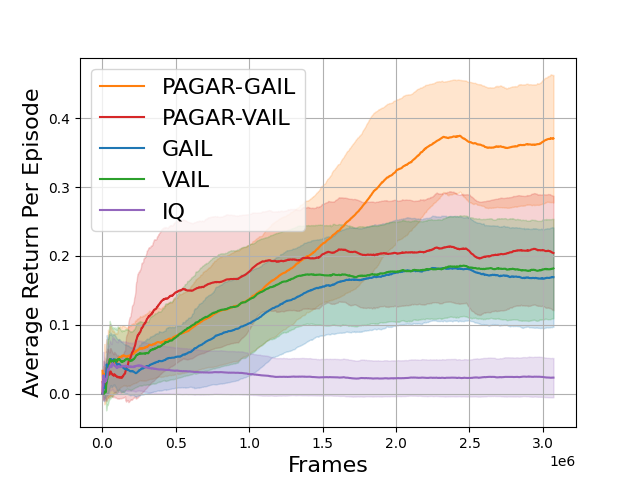}
        \caption{FourRooms}
    \end{subfigure}%
    \caption{ Comparing Algorithm \ref{alg:pagar2_1} with baselines in partial observable navigation tasks. 
    The suffix after each `{\tt{PAGAR-}}' indicates which IRL technique is used in Algorithm 1. 
    The $y$ axis indicates the average return per episode. The $x$ axis indicates the number of time steps.  }
    \label{fig:exp_fig2}
\end{figure*}

\section{Experiments}\label{sec:experiments}
The goal of our experiments is to assess
whether using PAGAR-based IL can efficiently mitigate reward misalignment in different IL/IRL benchmarks by comparing with representative baselines.
We present the main results below and provide details and additional results in Appendix \ref{sec:app_c}.
\subsection{Partially Observable Navigation Tasks} 
We first consider a maze navigation environment where the task objective can be straightforwardly categorized as either success or failure. 
Our benchmarks include two discrete domain tasks from the Mini-Grid environments~\cite{minigrid}: \textit{DoorKey-6x6-v0}, and \textit{SimpleCrossingS9N1-v0}. 
Due to partial observability and the implicit hierarchical nature of the task, these environments 
are considered challenging for RL and IL, and
have been extensively used for benchmarking
curriculum RL and exploration-driven RL.
In \textit{DoorKey-6x6-v0} the task is to pick up a key, unlock a door, and reach a target position; in \textit{SimpleCrossingS9N1}, the task is to pass an opening on a wall and reach a target position.   
The placements of the objects, obstacles, and doors are randomized in each instance of an environment.
The agent can only observe a small, unblocked area in front of it.
At each timestep, the agent can choose one out of $7$ actions, such as moving to the next cell or picking up an object.\li{such as?}
By default, the reward is always zero unless the agent reaches the target.
We compare our approach with two competitive baselines: GAIL~\cite{gail} and VAIL~\cite{vail}. 
GAIL has been introduced in Section \ref{sec:prelim}.
VAIL is based on GAIL but additionally optimizes a variational discriminator bottleneck (VDB) objective.   
Our approach uses the IRL techniques behind those two baseline algorithms, resulting in two versions of Algorithm \ref{alg:pagar2_1}, 
denoted as PAGAR-GAIL and PAGAR-VAIL, respectively.
More specifically, if the baseline optimizes a $J_{IRL}$ objective, we use the same $J_{IRL}$ objective in Algorithm \ref{alg:pagar2_1}.
Also, we represent the reward function $r$ with the discriminator $D$ as mentioned in Section \ref{sec:prelim}.
More details can be found in Appendix \ref{subsec:app_c_1}.
PPO~\cite{ppo} is used for policy training in GAIL, VAIL, and ours. 
Additionally, we compare our algorithm with a state-of-the-art (SOTA) IL algorithm, IQ-Learn~\cite{iq}, which, however, is not compatible with our algorithm because it does not explicitly optimize a reward function.
We use a replay buffer of size $2048$ in our algorithm and all the baselines. 
The policy and the reward functions are all approximated using convolutional networks.
By learning from $10$ expert-demonstrated trajectories with high returns, PAGAR-based IL produces high-performance policies with high sample efficiencies as shown in Figure \ref{fig:exp_fig2}(a) and (c). 
Furthermore, we compare PAGAR-VAIL with VAIL by reducing the number of demonstrations to one.
As shown in Figure \ref{fig:exp_fig2}(b) and (d), PAGAR-VAIL produces high-performance policies with significantly higher sample efficiencies. 
    
\noindent\textbf{Zero-Shot IL in Transfer Environments.}
In this experiment, we demonstrate that PAGAR enables the agent to infer and accomplish the objective of a task even in environments that are substantially different from the one observed during expert demonstrations.
As shown in Figure~\ref{fig:exp_fig2}(e), we collect $10$ expert demonstrations from the \textit{SimpleCrossingS9N1-v0} environment. 
Then we apply Algorithm~\ref{alg:pagar2_1} and the baselines, GAIL, VAIL, and IQ-learn to learn policies in \textit{SimpleCrossingS9N2-v0, SimpleCrossingS9N3-v0} and \textit{FourRooms-v0}.
The results in Figure \ref{fig:exp_fig2}(f)-(g) show that PAGAR-based IL outperforms the baselines
in these challenging zero-shot settings.


\noindent{\textbf{Influence of Reward Hypothesis Space}}.We study whether choosing a different reward function hypothesis set can influence the performance of Algorithm \ref{alg:pagar2_1}. 
Specifically, we compare using a $Sigmoid$ function with a Categorical distribution in the output layer of the discriminator networks in GAIL and PAGAR-GAIL.
When using the $Sigmoid$ function, the outputs of $D$ are not normalized, i.e., $\sum_{a\in \mathbb{A}}D(s,a)\neq 1$.
When using a Categorical distribution, the outputs in a state sum to one for all the actions, i.e., $\sum_{a\in \mathbb{A}}D(s,a)= 1$.
As a result, the sizes of the reward function sets are different in the two cases.
We test GAIL and PAGAR-GAIL in \textit{DoorKey-6x6-v0} environment.
\begin{wrapfigure}{l}{0.24\textwidth}
\vskip -0.2in
        \centering
        \includegraphics[width=1.\linewidth]{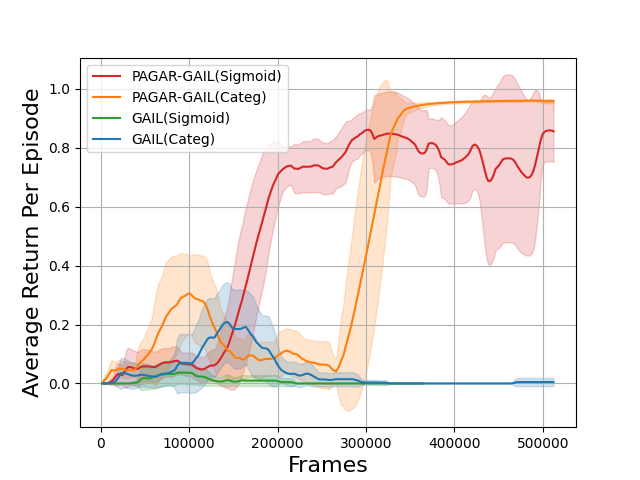}
        \vskip -0.1in
        \caption{Different reward spaces} 
\label{fig:exp_4}
\end{wrapfigure}
As shown in Figure \ref{fig:exp_4}, different reward function sets result in different training efficiency.
However, PAGAR-GAIL outperforms GAIL in both cases by using fewer samples to attain high performance.
\newline
\subsection{Continuous Tasks with Non-Binary Outcomes}
We test PAGAR-based IRL in $5$ Mujuco tasks where the task objectives do not have binary outcomes. 
In Figure \ref{fig:exp_fig1}, we show the results on two tasks (the other results are included in Appendix \ref{subsec:app_c_3}). 
The results show that PAGAR-based IL takes fewer iterations to achieve the same performance as the baselines. \li{not clear which one is more challenging if the reader is not familiar with the benchmarks}
In particular, in the \textit{HalfCheetah-v2} task, Algorithm \ref{alg:pagar2_1} achieves the same level of performance compared with GAIL and VAIL by using only half the numbers of iterations. 
\begin{figure}
\includegraphics[width=0.49\linewidth]{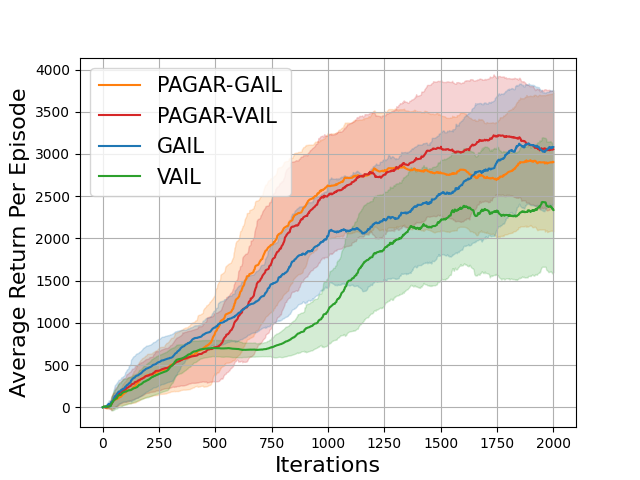}
\hfill
\includegraphics[width=0.49\linewidth]{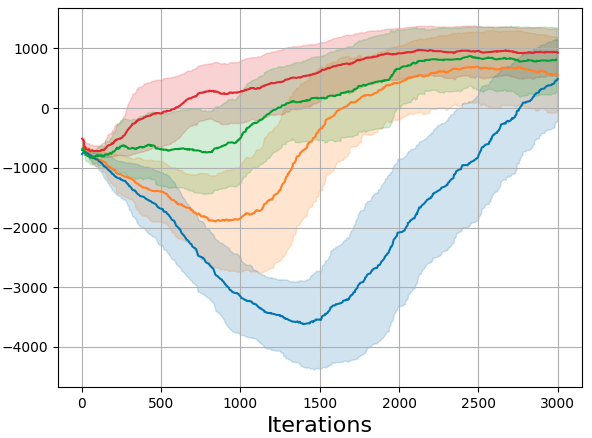}
\vspace{-.1in}
\caption{(Left: Walker2d-v2. Right: HalfCheeta-v2) The $y$ axis indicates the average return per episode\li{over how many runs?}. }
\label{fig:exp_fig1}
\vspace{-.1in}
\end{figure}
\section{Conclusion}
We propose PAGAR, a semi-supervised reward design paradigm that can mitigate the reward misalignment problem in IRL-based IL.
We present an on-and-off policy approach to PAGAR-based IL by using policy and reward improvement bounds to maximize the utilization of policy samples.
Experimental results demonstrate that our algorithm can mitigate reward misalignment in challenging environments.
Our future work will focus on accommodating the PAGAR paradigm in other alignment problems. 
\section{Impact Statement}
This paper presents work whose goal is to advance the field of Machine Learning. 
There are many potential societal consequences of our work, none of which we feel must be specifically highlighted here.
There are no identified ethical concerns associated with its implementation.
 
\nocite{langley00}

\bibliography{refs}
\bibliographystyle{icml2024}

\newpage
\appendix
\onecolumn
\section{Reward Design with PAGAR}\label{sec:app_a}
This paper does not aim to resolve the ambiguity problem in IRL but provides a way to circumvent it so that reward ambiguity does not lead to reward misalignment in IRL-based IL. 
PAGAR, the semi-supervised reward design paradigm proposed in this paper, tackles this problem from the perspective of semi-supervised reward design. 
But the nature of PAGAR is distinct from IRL and IL: assume that a set of reward functions is available for some underlying task, where some of those reward functions align with the task while others are misaligned, PAGAR provides a solution for selecting reward functions to train a policy that successfully performs the task, without knowing which reward function aligns with the task. 
Our research demonstrates that policy training with PAGAR is equivalent to learning a policy to maximize an affine combination of utilities measured under a distribution of the reward functions in the reward function set.
With this understanding of PAGAR, we integrate it with IL to illustrate its advantages.

\subsection{Semi-supervised Reward Design}\label{subsec:app_a_1}

Designing a reward function can be thought as deciding an ordering of policies. 
We adopt a concept, called \textit{total domination}, from unsupervised environment design~\cite{paired}, and re-interpret this concept in the context of reward design. 
In this paper, we suppose that the function $U_r(\pi)$ is given to measure the performance of a policy and it does not have to be the utility function. 
While the measurement of policy performance can vary depending on the free variable $r$, 
\textit{total dominance} can be viewed as an invariance regardless of such dependency.

\begin{definition}[Total Domination]
A policy, $\pi_1$, is totally dominated by some policy $\pi_2$ w.r.t a reward function set $R$, if for every pair of reward functions $r_1, r_2\in R$, $U_{r_1}(\pi_1) < U_{r_2}(\pi_2)$. 
\end{definition}

If $\pi_1$ totally dominate $\pi_2$ w.r.t $R$, $\pi_2$ can be regarded as being unconditionally better than $\pi_1$. 
In other words, the two sets $\{U_r(\pi_1)|r\in R\}$ and $\{U_r(\pi_2)|r\in R\}$ are disjoint, such that $\sup\{U_r(\pi_1)|r\in R\} < \inf\{U_r(\pi_2)|r\in R\}$. 
 Conversely, if a policy $\pi$ is not totally dominated by any other policy, it indicates that for any other policy, say $\pi_2$,  $\sup\{U_r(\pi_1)|r\in R\} \geq \inf\{U_r(\pi_2)|r\in R\}$.  

\begin{definition}
 A reward function set $R$ aligns with an ordering $\prec_R$ among policies such that $\pi_1\prec_R\pi_2$ if and only if $\pi_1$ is totally dominated by $\pi_2$ w.r.t. $R$.
\end{definition}

Especially, designing a reward function $r$ is to establish an ordering $\prec_{\{r\}}$ among policies. 
Total domination can be extended to policy-conditioned reward design, where the reward function $r$ is selected by following a decision rule $\mathcal{R}(\pi)$ such that $\sum_{r\in R}\mathcal{R}(\pi)(r)=1$.
We let $\mathcal{U}_{\mathcal{R}}(\pi)=\underset{r\in R}{\sum}\mathcal{R}(\pi)(r) \cdot U_r(\pi)$ be an affine combination of $U_r(\pi)$'s with its coefficients specified by $\mathcal{R}(\pi)$.

\begin{definition}
A policy conditioned decision rule $\mathcal{R}$ is said to prefer a policy $\pi_1$ to another policy $\pi_2$, which is notated as $\pi_1\prec^{\mathcal{R}} \pi_2$, if and only if $\mathcal{U}_{\mathcal{R}}(\pi_1)< \mathcal{U}_{\mathcal{R}}(\pi_2)$. 
\end{definition}

Making a decision rule for selecting reward functions from a reward function set to respect the total dominance w.r.t this reward function set is an unsupervised learning problem, where no additional external supervision is provided.  
If considering expert demonstrations as a form of supervision and using it to constrain the set ${R}_E$ of reward function via IRL, the reward design becomes semi-supervised.

\subsection{Solution to the MinimaxRegret}\label{subsec:app_a_2}
In Eq.\ref{eq:pagar1_2}, we mentioned that solving $MinimaxRegret(R)$ is equivalent to finding an optimal policy $\pi^*$ to maximize a $\mathcal{U}_{\mathcal{R}_E}(\pi)$ under a decision rule  $\mathcal{R}_E$.
Without loss of generality, we use $R$ instead of $R_{E,\delta}$ in our subsequent analysis, because solving $MinimaxRegret(R)$ does not depend on whether there are constraints for $R$.
In order to show such an equivalence, we follow the same routine as in \cite{paired}, and start by introducing the concept of \textit{weakly total domination}.

\begin{definition}[Weakly Total Domination]
A policy $\pi_1$ is \textit{weakly totally dominated} w.r.t a reward function set $R$ by some policy $\pi_2$ if and only if for any pair of reward function $r_1, r_2\in R$,  $U_{r_1}(\pi_1)  \leq U_{r_2}(\pi_2)$.
\end{definition}

Note that a policy $\pi$ being totally dominated by any other policy is a sufficient but not necessary condition for $\pi$ being weakly totally dominated by some other policy. A policy $\pi_1$ being weakly totally dominated by a policy $\pi_2$ implies that $\sup\{U_r(\pi_1)|r\in R\} \leq \inf\{U_r(\pi_2)|r\in R\}$. 
We assume that there does not exist a policy $\pi$ that weakly totally dominates itself, which could happen if and only if $U_r(\pi)$ is a constant. 
We formalize this assumption as the following.
\begin{assumption}\label{as:app_a_1}
For the given reward set $R$ and policy set $\Pi$, there does not exist a policy $\pi$ such that for any two reward functions $r_1, r_2\in R$, $U_{r_1}(\pi)=U_{r_2}(\pi)$.
\end{assumption}

This assumption makes weak total domination a non-reflexive relation. 
It is obvious that weak total domination is transitive and asymmetric. 
Now we show that successive weak total domination will lead to total domination. 

\begin{lemma}\label{lm:app_a_5}
for any three policies $\pi_1, \pi_2, \pi_3\in \Pi$, if $\pi_1$ is weakly totally dominated by $\pi_2$, $\pi_2$ is weakly totally dominated by $\pi_3$, then $\pi_3$ totally dominates $\pi_1$.
\end{lemma}
\begin{proof}
According to the definition of weak total domination, $\underset{r\in R}{\max}\ U_r(\pi_1)\leq \underset{r\in R}{\min}\ U_r(\pi_2)$ and $\underset{r\in R}{\max}\ U_r(\pi_2)\leq \underset{r\in R}{\min}\ U_r(\pi_3)$. 
If $\pi_1$ is weakly totally dominated but not totally dominated by $\pi_3$, then $\underset{r\in R}{\max}\ U_r(\pi_1)=\underset{r\in R}{\min}\ U_r(\pi_3)$ must be true. 
However, it implies $\underset{r\in R}{\min}\ U_r(\pi_2)=\underset{r\in R}{\max}\ U_r(\pi_2)$, which violates Assumption \ref{as:app_a_1}. 
We finish the proof.
\end{proof}

\begin{lemma}\label{lm:app_a_3}
For the set $\Pi_{\neg{wtd}}\subseteq \Pi$ of policies that are not weakly totally dominated by any other policy in the whole set of policies w.r.t a reward function set $R$, there exists a range $\mathbb{U}\subseteq\mathbb{R}$ such that for any policy $\pi\in \Pi_{\neg{wtd}}$, $\mathbb{U}\subseteq [\underset{r\in R}{\min}\ U_r(\pi), \underset{r\in R}{\max}\ U_r(\pi)]$.
\end{lemma}
\begin{proof}
For any two policies $\pi_1, \pi_2\in \Pi_{\neg{wtd}}$, it cannot be true that $\underset{r\in R}{\max}\ U_r(\pi_1) = \underset{r\in R}{\min}\ U_r(\pi_2)$ nor $\underset{r\in R}{\min}\ U_r(\pi_1) = \underset{r\in R}{\max}\ U_r(\pi_2)$, because otherwise one of the policies weakly totally dominates the other. Without loss of generalization, we assume that $\underset{r\in R}{\max}\ U_r(\pi_1) > \underset{r\in R}{\min}\ U_r(\pi_2)$. 
In this case, $\underset{r\in R}{\max}\ U_r(\pi_2)>\underset{r\in R}{\min}\ U_r(\pi_1)$ must also be true, otherwise $\pi_1$ weakly totally dominates $\pi_2$. Inductively, $\underset{\pi\in\Pi_{\neg{wtd}}}{\min}\ \underset{r\in R}{\max}\ U_r(\pi) > \underset{\pi\in\Pi_{\neg{wtd}}}{\max}\ \underset{r\in R}{\min}\ U_r(\pi)$. Letting $ub=\underset{\pi\in\Pi_{\neg{wtd}}}{\min}\ \underset{r\in R}{\max}\ U_r(\pi)$ and $lb=\underset{\pi\in\Pi_{\neg{wtd}}}{\max}\ \underset{r\in R}{\min}\ U_r(\pi)$, any $\mathbb{U}\subseteq [lb, ub]$ shall support the assertion.
We finish the proof.
\end{proof}
\begin{lemma}\label{lm:app_a_4}
For a reward function set $R$, if a policy $\pi\in\Pi$ is weakly totally dominated by some other policy in $\Pi$ and there exists a subset $\Pi_{\neg{wtd}}\subseteq \Pi$ of policies that are not weakly totally dominated by any other policy in $\pi$, then $\underset{r\in R}{\max}\ U_r(\pi) < \underset{\pi'\in\Pi_{\neg{wtd}}}{\min}\ \underset{r\in R}{\max}\ U_r(\pi')$
\end{lemma}
\begin{proof}
If $\pi_1$ is weakly totally dominated by a policy $\pi_2\in \Pi$, then $\underset{r\in R}{\min}\ U_r(\pi_2)=\underset{r\in R}{\max}\ U_r(\pi)$. If $\underset{r\in R}{\max}\ U_r(\pi) \geq \underset{\pi'\in\Pi_{\neg{wtd}}}{\min}\ \underset{r\in R}{\max}\ U_r(\pi')$, then $\underset{r\in R}{\min}\ U_r(\pi_2) \geq \underset{\pi'\in\Pi_{\neg{wtd}}}{\min}\ \underset{r\in R}{\max}\ U_r(\pi')$, making at least one of the policies in $\Pi_{\neg{wtd}}$ being weakly totally dominated by $\pi_2$. 
Hence, $\underset{r\in R}{\max}\ U_r(\pi) < \underset{\pi'\in\Pi_{\neg{wtd}}}{\min}\ \underset{r\in R}{\max}\ U_r(\pi')$ must be true.
\end{proof}

Given a policy $\pi$ and a reward function $r$, the regret is represented as Eq.\ref{eq:app_a_2}
\begin{eqnarray}
Regret(\pi, r)&:=& \underset{\pi'}{\max}\  U_r(\pi') - U_r(\pi)\label{eq:app_a_2}
\end{eqnarray}
Then we represent the $MinimaxRegret(R)$ problem in Eq.\ref{eq:app_a_3}.   

\begin{eqnarray}
MinimaxRegret(R)&:=& \arg\underset{\pi\in \Pi}{\min}\left\{\underset{r\in R}{\max}\ Regret(\pi, r)\right\}\label{eq:app_a_3}
\end{eqnarray}

We denote as $r^*_\pi\in R$ the reward function that maximizes $U_{r}(\pi)$ among all the $r$'s that achieve the maximization in Eq.\ref{eq:app_a_3}. Formally,
\begin{eqnarray}
    r^*_\pi&\in& \arg\underset{r\in R}{\max}\ U_r(\pi)\qquad s.t.\ r\in \arg\underset{r'\in R}{\max}\ Regret(\pi, r')
\end{eqnarray}

Then $MinimaxRegret$ can be defined as minimizing the worst-case regret as in Eq.\ref{eq:app_a_3}. 
Next, we want to show that for some decision rule $\mathcal{R}$, the set of optimal policies which maximizes $\mathcal{U}_{\mathcal{R}}$ are the solutions to $MinimaxRegret(R)$. 
Formally, 
\begin{eqnarray}
    MinimaxRegret(R) = \arg\underset{\pi\in \Pi}{\max}\ \mathcal{U}_{\mathcal{R}}(\pi)\label{eq:app_a_5}
\end{eqnarray}

We design $\mathcal{R}$ by letting $\mathcal{R}(\pi):= \overline{\mathcal{R}}(\pi)\cdot \delta_{r^*_\pi}  +  (1-\overline{\mathcal{R}}(\pi))\cdot \underline{\mathcal{R}}(\pi)$ where  $\underline{\mathcal{R}}: \Pi\rightarrow \Delta(R)$ is a policy conditioned distribution over reward functions, $\delta_{r^*_\pi}$ be a delta distribution centered at $r^*_\pi$, and $\overline{\mathcal{R}}(\pi)$ is a coefficient. 
We show how to design $\underline{\mathcal{R}}$ by using the following lemma. 
 
\begin{lemma}\label{lm:app_a_1}
Given that the reward function set is $R$, there exists a decision rule $\underline{\mathcal{R}}: \Pi\rightarrow \Delta(R)$ which guarantees that:  
1) for any policy $\pi$ that is not weakly totally dominated by any other policy in $\Pi$, i.e., $\pi\in\Pi_{\neg wtd}\subseteq\Pi$, $\mathcal{U}_{\underline{\mathcal{R}}}(\pi)\equiv c$ where $c=\underset{\pi'\in\Pi_{\neg wtd}}{\max}\ \underset{r\in R}{\min}\ U_r(\pi')$; 2) for any $\pi$ that is weakly totally dominated by some policy but not totally dominated by any policy, $\mathcal{U}_{\underline{\mathcal{R}}}(\pi)=\underset{r\in R}{\max}\ U_r(\pi)$; 3) if $\pi$ is totally dominated by some other policy, $\overline{\mathcal{R}}(\pi)$ is a uniform distribution.
\end{lemma}
\begin{proof}
Since the description of $\underline{\mathcal{R}}$ for the policies in condition 2) and 3) are self-explanatory, we omit the discussion on them. 
For the none weakly totally dominated policies in condition 1), having a constant
$\mathcal{U}_{\underline{\mathcal{R}}}(\pi)\equiv c$ is possible if and only if for any policy $\pi\in\Pi_{\neg wed}$, $c\in[\underset{r\in R}{\min}\ U_r(\pi'), \underset{r\in R}{\max}\ U_r(\pi')]$. 
As mentioned in the proof of Lemma \ref{lm:app_a_3}, $c$ can exist within $[\underset{r\in R}{\min}\ U_r(\pi), \underset{r\in R}{\max}\ U_r(\pi)]$.
Hence, $c=\underset{\pi'\in\Pi_{\neg wtd}}{\max}\ \underset{r\in R}{\min}\ U_r(\pi')$ is a valid assignment.
\end{proof}

Then by letting $\overline{\mathcal{R}}(\pi):=\frac{Regret(\pi, r^*_\pi)}{ c - U_{r^*_\pi}(\pi) }$, we have the following theorem.

\begin{theorem}\label{th:app_a_1}
By letting $\mathcal{R}(\pi):= \overline{\mathcal{R}}(\pi)\cdot \delta_{r^*_\pi}  +  (1-\overline{\mathcal{R}}(\pi))\cdot \underline{\mathcal{R}}(\pi)$ with $\overline{\mathcal{R}}(\pi):=\frac{Regret(\pi, r^*_\pi)}{ c - U_{r^*_\pi}(\pi) }$ and any $\underline{\mathcal{R}}$ that satisfies Lemma \ref{lm:app_a_1},
\begin{eqnarray}
     MinimaxRegret(R) = \arg\underset{\pi\in \Pi}{\max}\ \mathcal{U}_{\mathcal{R}}(\pi)
\end{eqnarray}
\end{theorem}
\begin{proof}
If a policy $\pi\in\Pi$ is totally dominated by some other policy, since there exists another policy with larger $\mathcal{U}_{\mathcal{R}}$, $\pi$ cannot be a solution to $\arg\underset{\pi\in \Pi}{\max}\ \mathcal{U}_{\mathcal{R}}(\pi)$. 
Hence, there is no need for further discussion on totally dominated policies. We discuss the none weakly totally dominated policies and the weakly totally dominated but not totally dominated policies (shortened to "weakly totally dominated" from now on) respectively. First we expand  $\arg\underset{\pi\in\Pi}{\max}\ \mathcal{U}_{\mathcal{R}}(\pi)$ as in Eq.\ref{eq:app_a_7}.
\begin{eqnarray}
        && \arg\underset{\pi\in\Pi}{\max}\ \mathcal{U}_{\mathcal{R}}(\pi)\nonumber\\
        &=& \arg\underset{\pi\in\Pi}{\max}\ \underset{r\in R}{\sum}\mathcal{R}(\pi)(r)\cdot U_r(\pi)\nonumber\\
        &=& \arg\underset{\pi\in\Pi}{\max}\ \frac{Regret(\pi, r^*_\pi)\cdot U_{r^*_\pi} (\pi)  +   (\mathcal{U}_{\underline{\mathcal{R}}}(\pi)-U_{r^*_\pi}(\pi) - Regret(\pi, r^*_\pi)) \cdot \mathcal{U}_{\underline{\mathcal{R}}}(\pi)}{c-U_{r^*_\pi}(\pi)}\nonumber\\
        &=&  \arg\underset{\pi\in\Pi}{\max}\ \frac{(\mathcal{U}_{\underline{\mathcal{R}}}(\pi)-U_{r^*_\pi}(\pi))\cdot \mathcal{U}_{\underline{\mathcal{R}}}(\pi) - (\mathcal{U}_{\underline{\mathcal{R}}}(\pi) - U_{r^*_\pi} (\pi))\cdot Regret(\pi, r^*_\pi))}{c-U_{r^*_\pi}(\pi)}\nonumber\\
        &=& \arg\underset{\pi\in\Pi}{\max}\ \frac{\mathcal{U}_{\underline{\mathcal{R}}}(\pi)-U_{r^*_\pi}(\pi)}{c-U_{r^*_\pi}(\pi)} \cdot \mathcal{U}_{\underline{\mathcal{R}}}(\pi) -  Regret(\pi, r^*_\pi)\label{eq:app_a_7}
    \end{eqnarray}
1) For the none weakly totally dominated policies, since by design $\mathcal{U}_{\underline{\mathcal{R}}}\equiv c$, Eq.\ref{eq:app_a_7} is equivalent to $\arg\underset{\pi\in \Pi_1}{\max}\ - Regret(\pi, r^*_\pi)$ which exactly equals $MinimaxRegret(R)$. Hence, the equivalence holds among the none weakly totally dominated policies. Furthermore, if a none weakly totally dominated policy $\pi\in\Pi_{\neg wtd}$ achieves optimality in $MinimaxRegret(R)$, its $\mathcal{U}_{\mathcal{R}}(\pi)$ is also no less than any weakly totally dominated policy. 
Because according to Lemma \ref{lm:app_a_4}, for any weakly totally dominated policy $\pi_1$, its $\mathcal{U}_{\underline{\mathcal{R}}}(\pi_1)\leq c$, hence $\frac{\mathcal{U}_{\underline{\mathcal{R}}}(\pi)-U_{r^*_\pi}(\pi)}{c-U_{r^*_\pi}(\pi)}\cdot \mathcal{U}_{\underline{\mathcal{R}}}(\pi_1)\leq c$. Since $Regret(\pi, r^*_\pi)\leq Regret(\pi_1, r^*_{\pi_1})$, $\mathcal{U}_{\mathcal{R}}(\pi)\geq \mathcal{U}_{\mathcal{R}}(\pi_1)$. 
Therefore, we can assert that if a none weakly totally dominated policy $\pi$ is a solution to $MinimaxRegret(R)$, it is also a solution to $\arg\underset{\pi\in\Pi}{\max}\ \mathcal{U}_{\mathcal{R}}(\pi)$. 
Additionally, to prove that if a none weakly totally dominated policy $\pi$ is a solution to $\arg\underset{\pi'\in\Pi}{\max}\ \mathcal{U}_{\mathcal{R}}(\pi')$, it is also a solution to $MinimaxRegret(R)$, it is only necessary to prove that $\pi$ achieve no larger regret than all the weakly totally dominated policies. 
But we delay the proof to 2).

2) If a policy $\pi$ is weakly totally dominated and is a solution to $MinimaxRegret(R)$, we show that it is also a solution to $\arg\underset{\pi\in\Pi}{\max}\ \mathcal{U}_{\mathcal{R}}(\pi)$, i.e., its $\mathcal{U}_{\mathcal{R}}(\pi)$ is no less than that of any other policy. 

We start by comparing with non weakly totally dominated policy. 
for any weakly totally dominated policy $\pi_1\in MinimaxRegret(R)$, it must hold true that $Regret(\pi_1, r^*_{\pi_1})\leq Regret(\pi_2, r^*_{\pi_2})$ for any $\pi_2\in \Pi$ that weakly totally dominates $\pi_1$. 
However, it also holds that $Regret(\pi_2, r^*_{\pi_2})\leq Regret(\pi_1, r^*_{\pi_2})$ due to the weak total domination. 
Therefore, $Regret(\pi_1, r^*_{\pi_1})= Regret(\pi_2, r^*_{\pi_2})=Regret(\pi_1, r^*_{\pi_2})$, implying that $\pi_2$ is also a solution to $MinimaxRegret(R)$.
It also implies that $U_{r^*_{\pi_2}}(\pi_1)=U_{r^*_{\pi_2}}(\pi_2)\geq U_{r^*_{\pi_1}}(\pi_1)$ due to the weak total domination. 
However, by definition $U_{r^*_{\pi_1}}(\pi_1)\geq U_{r^*_{\pi_2}}(\pi_1)$. 
Hence, $U_{r^*_{\pi_1}}(\pi_1)= U_{r^*_{\pi_2}}(\pi_1)=U_{r^*_{\pi_2}}(\pi_2)$ must hold. 
Now we discuss two possibilities: a) there exists another policy $\pi_3$ that weakly totally dominates $\pi_2$; b) there does not exist any other policy that weakly totally dominates $\pi_2$.
First, condition a) cannot hold. 
Because inductively it can be derived $U_{r^*_{\pi_1}}(\pi_1)= U_{r^*_{\pi_2}}(\pi_1)=U_{r^*_{\pi_2}}(\pi_2)=U_{r^*_{\pi_3}}(\pi_3)$, while Lemma \ref{lm:app_a_5} indicates that $\pi_3$ totally dominates $\pi_1$, which is a contradiction. 
Hence, there does not exist any policy that weakly totally dominates $\pi_2$, meaning that condition b) is certain. 
We note that $U_{r^*_{\pi_1}}(\pi_1)= U_{r^*_{\pi_2}}(\pi_1)=U_{r^*_{\pi_2}}(\pi_2)$ and the weak total domination between $\pi_1, \pi_2$ imply that $r^*_{\pi_1}, r^*_{\pi_2}\in \arg\underset{r\in R}{\max}\ U_{r}(\pi_1)$, $r^*_{\pi_2}\in \arg\underset{r\in R}{\min}\ U_{r}(\pi_2)$, and thus $\underset{r\in R}{\min}\ U_{r}(\pi_2)\leq \underset{\pi \in \Pi_{\neg wtd}}{\max}\ \underset{r\in R}{\min}\ U_r(\pi)=c$. 
Again, $\pi_1\in MinimaxRegret(R)$ makes $Regret(\pi_1, r^*_{\pi})\leq Regret(\pi_1, r^*_{\pi_1})\leq Regret(\pi, r^*_{\pi})$ not only hold for $\pi=\pi_2$ but also for any other policy $\pi\in\Pi_{\neg wtd}$, then for any policy $\pi\in\Pi_{\neg wtd}$, $U_{r^*_{\pi}}(\pi_1)\geq U_{r^*_\pi}(\pi) \geq \underset{r\in R}{\min}\ U_r(\pi)$. 
Hence, $U_{r^*_{\pi}}(\pi_1)\geq \underset{\pi \in \Pi_{\neg wtd}}{\max}\ \underset{r\in R}{\min}\ U_r(\pi)=c$. 
Since $U_{r^*_{\pi}}(\pi_1)=\underset{r\in R}{\min}\ U_{r}(\pi_2)$ as aforementioned, $\underset{r\in R}{\min}\ U_{r}(\pi_2) > \underset{\pi \in \Pi_{\neg wtd}}{\max}\ \underset{r\in R}{\min}\ U_r(\pi)$ will cause a contradiction. Hence, $\underset{r\in R}{\min}\ U_{r}(\pi_2) =\underset{\pi \in \Pi_{\neg wtd}}{\max}\ \underset{r\in R}{\min}\ U_r(\pi)=c$. 
As a result, $\mathcal{U}_{\underline{\mathcal{R}}}(\pi)=U_{r^*_{\pi}}(\pi)=\underset{\pi' \in \Pi_{\neg wtd}}{\max}\ \underset{r\in R}{\min}\ U_r(\pi')=c$, and $\mathcal{U}_{\mathcal{R}}(\pi)=c- Regret(\pi, r^*_\pi)\geq  \underset{\pi' \in \Pi_{\neg wtd}}{\max}\ c - Regret(\pi', r^*_{\pi'})=\underset{\pi' \in \Pi_{\neg wtd}}{\max}\ \mathcal{U}_{\mathcal{R}}(\pi')$. 
In other words, if a weakly totally dominated policy $\pi$ is a solution to $MinimaxRegret(R)$, then its $\mathcal{U}_{\mathcal{R}}(\pi)$ is no less than that of any non weakly totally dominated policy. This also complete the proof at the end of 1), because if a none weakly totally dominated policy $\pi_1$ is a solution to $\arg\underset{\pi\in\Pi}{\max}\ \mathcal{U}_{\mathcal{R}}(\pi)$ but not a solution to $MinimaxRegret(R)$, then $Regret(\pi_1, r^*_{\pi_1})>0$ and a weakly totally dominated policy $\pi_2$ must be the solution to $MinimaxRegret(R)$. 
Then, $\mathcal{U}_{\mathcal{R}}(\pi_2)=c > c-Regret(\pi_1, r^*_{\pi_1})=\mathcal{U}_{\mathcal{R}}(\pi_1)$, which, however, contradicts $\pi_1\in \arg\underset{\pi\in\Pi}{\max}\ \mathcal{U}_{\mathcal{R}}(\pi)$.

It is obvious that a weakly totally dominated policy $\pi\in MinimaxRegret(R)$ has a $\mathcal{U}_{\mathcal{R}}(\pi)$ no less than any other weakly totally dominated policy. Because for any other weakly totally dominated policy $\pi_1$, $\mathcal{U}_{\underline{\mathcal{R}}}(\pi_1)\leq c$ and $Regret(\pi_1, r^*_{\pi_1})\leq Regret(\pi, r^*_\pi)$, hence $\mathcal{U}_{\mathcal{R}}(\pi_1)\leq \mathcal{U}_{\mathcal{R}}(\pi)$ according to Eq.\ref{eq:app_a_7}.

So far we have shown that if a weakly totally dominated policy $\pi$ is a solution to $MinimaxRegret(R)$, it is also a solution to $\arg\underset{\pi'\in\Pi}{\max}\ \mathcal{U}_{\mathcal{R}}(\pi')$. 
Next, we need to show that the reverse is also true, i.e., if a weakly totally dominated policy $\pi$ is a solution to $\arg\underset{\pi\in\Pi}{\max}\ \mathcal{U}_{\mathcal{R}}(\pi)$, it must also be a solution to $MinimaxRegret(R)$. 
In order to prove its truthfulness, we need to show that if $\pi\notin MinimaxRegret(R)$, whether there exists: a) a none weakly totally dominated policy $\pi_1$, or b) another weakly totally dominated policy $\pi_1$, such that $\pi_1\in MinimaxRegret(R)$ and $\mathcal{U}_{\mathcal{R}}(\pi_1)\leq \mathcal{U}_{\mathcal{R}}(\pi)$. 
If neither of the two policies exists, we can complete our proof. 
Since it has been proved in 1) that if a none weakly totally dominated policy achieves $MinimaxRegret(R)$, it also achieves $\arg\underset{\pi'\in\Pi}{\max}\ \mathcal{U}_{\mathcal{R}}(\pi')$, the policy described in condition a) does not exist. 
Hence, it is only necessary to prove that the policy in condition b) also does not exist.

If such weakly totally dominated policy $\pi_1$ exists,   $\pi\notin MinimaxRegret(R)$ and $\pi_1\in MinimaxRegret(R)$ indicates $Regret(\pi, r^*_{\pi}) > Regret(\pi_1, r^*_{\pi_1})$. 
Since $\mathcal{U}_{\mathcal{R}}(\pi_1)\geq \mathcal{U}_{\mathcal{R}}(\pi)$, according to Eq.\ref{eq:app_a_7}, $\mathcal{U}_{\mathcal{R}}(\pi_1)=c - Regret(\pi_1, r^*_{\pi_1})\leq \mathcal{U}_{\mathcal{R}}(\pi)=\frac{\mathcal{U}_{\underline{\mathcal{R}}}(\pi)-U_{r^*_\pi}(\pi)}{c-U_{r^*_\pi}(\pi)} \cdot \mathcal{U}_{\underline{\mathcal{R}}}(\pi) - Regret(\pi, r^*_\pi)$. 
Thus $\frac{\mathcal{U}_{\underline{\mathcal{R}}}(\pi)-U_{r^*_\pi}(\pi)}{c-U_{r^*_\pi}(\pi)}(\pi) \cdot \mathcal{U}_{\underline{\mathcal{R}}} \geq  c + Regret(\pi, r^*_{\pi}) - Regret(\pi_1, r^*_{\pi_1}) > c$, which is impossible due to $\mathcal{U}_{\underline{\mathcal{R}}}\leq c$.
Therefore, such $\pi_1$ also does not exist. 
In fact, this can be reasoned from another perspective. 
If there exists a weakly totally dominated policy $\pi_1$ with $U_{r^*_{\pi_1}}(\pi_1)=c=U_{r^*_\pi}(\pi)$ but $\pi_1\notin MinimaxRegret(R)$, then $Regret(\pi, r^*_{\pi}) > Regret(\pi_1, r^*_{\pi_1})$. 
It also indicates $\underset{\pi'\in\Pi}{\max}\ U_{r^*_{\pi}}(\pi') > \underset{\pi'\in\Pi}{\max}\ U_{r^*_{\pi_1}}(\pi')$. 
Meanwhile, $Regret(\pi_1, r^*_{\pi}):=\underset{\pi'\in\Pi}{\max}\ U_{r^*_{\pi}}(\pi') - U_{r^*_{\pi}}(\pi_1) \leq  Regret(\pi_1, r^*_{\pi_1}):= \underset{\pi'\in\Pi}{\max}\ U_{r^*_{\pi_1}}(\pi') - U_{r^*_{\pi_1}}(\pi_1):= \underset{r\in R}{\max}\ \underset{\pi'\in\Pi}{\max}\ U_{r}(\pi') - U_{r}(\pi_1)$ indicates $\underset{\pi'\in\Pi}{\max}\ U_{r^*_{\pi}}(\pi') - \underset{\pi'\in\Pi}{\max}\ U_{r^*_{\pi_1}}(\pi') \leq U_{r^*_{\pi}}(\pi_1) - U_{r^*_{\pi_1}}(\pi_1)$. 
However, we have proved that, for a weakly totally dominated policy, $\pi_1 \in MinimaxRegret(R)$ indicates $U_{r^*_{\pi_1}}(\pi_1)=\underset{r\in R}{\max}\ U_r(\pi_1)$. 
Hence, $\underset{\pi'\in\Pi}{\max}\ U_{r^*_{\pi}}(\pi') - \underset{\pi'\in\Pi}{\max}\ U_{r^*_{\pi_1}}(\pi') \leq U_{r^*_{\pi}}(\pi_1) - U_{r^*_{\pi_1}}(\pi_1)\leq 0$ and it contradicts $\underset{\pi'\in\Pi}{\max}\ U_{r^*_{\pi}}(\pi') > \underset{\pi'\in\Pi}{\max}\ U_{r^*_{\pi_1}}(\pi')$. Therefore, such $\pi_1$ does not exist. 
In summary, we have exhausted all conditions and can assert that for any policies, being a solution to $MinimaxRegret(R)$ is equivalent to a solution to $\arg\underset{\pi\in \Pi}{\max}\ \mathcal{U}_{\mathcal{R}}(\pi)$. We complete our proof.

\end{proof}

\subsection{Collective Validation of Similarity Between Expert and Agent}\label{subsec:app_a_3}

In Definition \ref{def:sec1_1} and our definition of $Regret$ in Eq.\ref{eq:pagar1_1}, we use the utility function $U_r$ to measure the performance of a policy. 
We now show that we can replace $U_r$ with other functions.
\begin{lemma}\label{lm:app_a_6}
The solution of $MinimaxRegret(R_{E,\delta^*})$ does not change when $U_r$ in $MinimaxRegret$ is replace with $U_r(\pi) - f(r)$ where $f$ can be arbitrary function of $r$.
\end{lemma}
\begin{proof}
When using $U_r(\pi) - f(r)$ instead of $U_r(\pi)$ to measure the policy performance, solving $MinimaxRegret(R)$ is to solve Eq.
\ref{eq:app_a_8}, which is the same as Eq.\ref{eq:app_a_3}.
\begin{eqnarray}
MimimaxRegret(R)&=&\arg\underset{\pi\in\Pi}{\max}\ \underset{r\in R}{\min}\ Regret(\pi, r)\nonumber\\
&=&\arg\underset{\pi\in\Pi}{\max}\ \underset{r\in R}{\min}\ \underset{\pi'\in\Pi}{\max}\left\{U_r(\pi')-f(r)\right\} - (U_r(\pi) -f(r))\nonumber\\
&=& \arg\underset{\pi\in\Pi}{\max}\ \underset{r\in R}{\min}\ \underset{\pi'\in\Pi}{\max}\ U_r(\pi') - U_r(\pi)\label{eq:app_a_8}
\end{eqnarray}
\end{proof}

Lemma \ref{lm:app_a_6} implies that we can use the policy-expert margin $U_r(\pi)-U_r(E)$ as a measurement of policy performance.
This makes the rationale of using PAGAR-based IL for collective validation of similarity between $E$ and $\pi$ more intuitive.

\subsection{Reward Misalignment in IRL-based IL}\label{subsec:app_a_5}
\noindent{\textbf{Lemma \ref{lm:sec4_1}.} 
    The optimal solution $r^*$ of IRL is misaligned with the task specified by $\Phi$ iff $\Phi(\pi_{r^*})\equiv false$. 
\begin{proof}
    If $r^*$ is misaligned, there does not exist an $S_{r^*}$ as defined in Definition \ref{def:sec1_1}. 
    Then, $\Phi(\pi_{r^*})=false$.
    Otherwise, there exists a singleton $S_{r^*}=[U_{r^*}(\pi_{r^*}), U_{r^*}(\pi_{r^*})]$, which is contradicting.
    If $\Phi(\pi_{r^*})\equiv false$, there does not exist an $S_{r^*}$ as well.
    Otherwise, according to Definition \ref{def:sec1_1}, $U_{r^*}(\pi_{r^*})\in S_{r^*}$, which is contradicting.
    In conclusion, the optimal reward function $r^*$ of IRL being misaligned with the task specified by $\Phi$ and $\Phi(\pi_{r^*})\equiv false$ are necessary and sufficient conditions for each other.
\end{proof}

\subsection{Criterion for Successful Policy Learning}\label{subsec:app_a_4}
\noindent{\textbf{Theorem  {\ref{th:pagar_1_1}.}}({Task-Failure Avoidance}\li{of what?}) 
If the following conditions (1) (2) hold for $R$, then the optimal protagonist policy $\pi_P:=MinimaxRegret(R)$ satisfies that $\forall {r^+}\in R$,$ U_{{r^+}}(\pi_P)\notin F_{{r^+}}$.  
\begin{enumerate}
\item[(1)] There exists ${r^+}\in R$, and $\underset{{r^+}\in R}{\max}\ \{\sup F_{{r^+}} - \inf F_{{r^+}}\} < \underset{{{r^+}}\in R}{\min}\ \{\inf S_{{r^+}} - \sup F_{{r^+}}\} \wedge \underset{{{r^+}}\in R}{\max}\ \{\sup S_{{r^+}} - \inf S_{{r^+}}\} < \underset{{{r^+}}\in R}{\min}\ \{\inf S_{{r^+}} - \sup F_{{r^+}}\}$; \\
\item[(2)]  There exists a policy $\pi^*$ such that $\forall {r^+}\in R$, $U_{{r^+}}(\pi^*)\in S_{{r^+}}$, and $\forall r^-\in R$,  $\underset{\pi\in\Pi}{\max}\ U_{r^-}(\pi) - U_{r^-}(\pi^*)  <  \underset{{{r^+}}\in R}{\min}\ \{\inf S_{{r^+}} - \sup F_{{r^+}}\}$.
\end{enumerate}
}
\begin{proof}
Suppose the conditions are met, and a policy $\pi_1$ satisfies the property described in conditions 2). Then for any policy $\pi_2\in MinimaxRegret(R)$, if $\pi_2$ does not satisfy the mentioned property, there exists a task-aligned reward function ${r^+}\in R$ such that  $U_{{r^+}}(\pi_2)\in F_{{r^+}}$. 
In this case $Regret(\pi_2, {{r^+}})=\underset{\pi\in\Pi}{\max}\ U_{{r^+}}(\pi) - U_{{r^+}}(\pi_2)\geq \inf S_{{r^+}} - \sup F_{{r^+}}\geq \underset{{{r^+}'}\in R}{\min}\ \left\{\inf S_{{r^+}'} - \sup F_{{{r^+}'}}\right\}$. However, for $\pi_1$, it holds for any task-aligned reward function $\hat{r}_{al}\in R$ that $Regret(\pi_2, {\hat{r}_{al}})\leq \sup S_{\hat{r}_{al}} - \inf S_{\hat{r}_{al}}<\underset{{{r^+}'}\in R}{\min}\ \left\{\inf S_{{r^+}'} - \sup F_{{r^+}'}\right\}$, and it also holds for any misaligned reward function $r^-\in R$ that $Regret(\pi_2, _{r^-})=\underset{\pi\in\Pi}{\max}\ U_{r^-}(\pi) - U_{r^-}(\pi_2) <  \underset{{{r^+}'}\in R}{\min}\ \left\{\inf S_{{r^+}} - \sup F_{{r^+}}\right\}$. Hence, $Regret(\pi_2, {{r^+}})< Regret(\pi_1, {{r^+}})$, contradicting $\pi_1\in MiniRegret$. We complete the proof.
\end{proof}

\noindent{\textbf{Theorem  {\ref{th:pagar_1_0}.}}(Task-Success Guarantee)
Assume that Condition (1) in Theorem \ref{th:pagar_1_1} is satisfied. 
If there exists a policy $\pi^*$ such that $\forall r\in R$, $\underset{\pi}{\max}\ U_{r}(\pi) - U_{r}(\pi^*)  <  \underset{{{r^+}}\in R}{\min}\ \{\sup S_{{r^+}} - \inf S_{{r^+}}\}$, then the optimal protagonist policy $\pi_P:=MinimaxRegret(R)$ satisfies that $\forall {r^+}\in R$, $U_{{r^+}}(\pi_P)\in S_{{r^+}}$.
}
\begin{proof}
Since $\underset{r\in R}{max}\ \underset{\pi}{\max}\ U_r(\pi)-U_r(\pi_P)\leq \underset{r\in R}{max}\ \underset{\pi}{\max}\ U_r(\pi)-U_r(\pi^*)< \underset{{{r^+}}\in R}{\min}\ \{\sup S_{{r^+}} - \inf S_{{r^+}}\}$, we can conclude that for any ${r^+}\in R$, $\underset{\pi}{\max}\ U_{{r^+}}(\pi)-U_{{r^+}}(\pi_P)\leq \{\sup S_{{r^+}} - \inf S_{{r^+}}$, in other words, $U_{{r^+}}(\pi_P)\in S_{{r^+}}$.
The proof is complete.
\end{proof}

\noindent{\textbf{Corollary {\ref{th:pagar_1_3}.}} 
If the following conditions (1) (2) hold for $R_{E,\delta}$, then the optimal protagonist policy $\pi_P:=MinimaxRegret(R_{E,\delta})$ satisfies that $\forall {r^+}\in R_{E,\delta}$, $U_{{r^+}}(\pi_P)\notin F_{{r^+}}$ .
\begin{enumerate}
\item[(1)] The condition (1) in Theorem \ref{th:pagar_1_1} holds
\item[(2)] $\forall {{r^+}}\in R_{E,\delta}$, $L_{{r^+}}\cdot W_E-\delta \leq \sup S_{{r^+}} - \inf S_{{r^+}}$ and $\forall r^-\in R_{E,\delta}$,  $L_{r^-}\cdot W_E-\delta  <  \underset{{{r^+}}\in R_{E,\delta}}{\min}\ \{\inf S_{{r^+}} - \sup F_{{r^+}}\}$. 
\end{enumerate}
}
\begin{proof}
We consider $U_r(\pi)=\mathbb{E}_{\tau\sim \pi}[r(\tau)]$. 
Since $W_E\triangleq \underset{\pi\in\Pi}{\min}\ W_1(\pi,E)={\frac {1}{K}}\underset{\|r\|_{L}\leq K}{\sup}\ U_r(E)-U_r(\pi)$ for any $K>0$, let $\pi^*$ be the policy that achieves the minimality in $W_E$. Then for any ${r^+}\in R$, the term $L_{{r^+}}\cdot W_E-\delta \geq L_{{r^+}}\cdot \frac{1}{L_{{r^+}}} \underset{\|r\|_{L}\leq L_{{r^+}}}{\sup}\ U_r(E)-U_r(\pi)\geq U_{{r^+}}(E)-U_{{r^+}}(\pi)$.
Hence, for all ${r^+}\in R$, $U_{{r^+}}(E)-U_{{r^+}}(\pi) < \sup S_{{r^+}} - \inf S_{{r^+}}$, i.e., $U_{{r^+}}(\pi^*)\in S_{{r^+}}$.
Likewise, $L_{r^-}\cdot W_E-\delta  <  \underset{{{r^+}}\in R_{E,\delta}}{\min}\ \{\inf S_{{r^+}} - \sup F_{{r^+}}\}$ indicates that for all $r^-\in R$, $U_{r^-}(E)-U_{r^-}(\pi) < \underset{{{r^+}}\in R_{E,\delta}}{\min}\ \{\inf S_{{r^+}} - \sup F_{{r^+}}\}$.
Then, we have recovered the condition (2) in Theorem \ref{th:pagar_1_1}.
As a result, we deliver the same guarantees in Theorem \ref{th:pagar_1_1}.
\end{proof}

\noindent{\textbf{Corollary {\ref{th:pagar_1_4}}.} 
Assume that the condition (1) in Theorem \ref{th:pagar_1_1} holds for $R_{E,\delta}$.
If for any $r\in R_{E,\delta}$, $L_r\cdot W_E-\delta \leq  \underset{{{r^+}}\in R}{\min}\ \{\sup S_{{r^+}} - \inf S_{{r^+}}\}$, then the optimal protagonist policy $\pi_P=MinimaxRegret(R_{E,\delta})$ satisfies \li{not sure if there is a missing connective here}$\forall {r^+}\in R_{E,\delta}$, $U_{{r^+}}(\pi_P)\in S_{{r^+}}$.
 }
\begin{proof}
Again, we let $\pi^*$ be the policy that achieves the minimality in $W_E$.
Then, we have $L_r\cdot W_E-\delta \geq L_r\cdot \frac{1}{L_r} \underset{\|r\|_{L}\leq L_r}{\sup}\ U_r(E)-U_r(\pi^*)\geq U_r(E)-U_r(\pi^*)$ for any $r\in R$.
We have recovered the condition in Corollary \ref{th:pagar_1_0}.
The proof is complete.
\end{proof}

Recall that the IRL loss be in the form of $J_{IRL}(r):=U_r(E)-\underset{\pi\in\Pi}{\max}\ U_r(\pi)$. 
We assume that $\arg\underset{r\in R}{\min}\ J_{IRL}(r)$ can reach Nash Equilibrium with a reward function set $R_{E,\delta^*}$ and a policy set $\Pi_E$.

\noindent{\textbf{Proposition {\ref{th:pagar_1_2}.}}
$\Pi_E:=MinimiaxRegret(R_{E,\delta^*})$.
}
\begin{proof}
The reward function set $R_{E,\delta^*}$ and the policy set $\Pi_E$ achieving Nash Equilibrium for $\arg\underset{r\in R}{\min}\ J_{IRL}(r)$ indicates that for any $r\in R_{E,\delta^*}, \pi\in\Pi_E$, $\pi \in \arg\underset{\pi\in\Pi}{\max}\ U_r(\pi) - U_r(E)$. 
Then $\Pi_E$ will be the solution to $\arg\underset{\pi_P\in \Pi}{\max}\ \underset{r\in R_{E,\delta^*}}{\min}\ \left\{\underset{\pi_A\in \Pi}{\max}\ U_r(\pi_A) - U_r(E)\right\} - (U_r(\pi_P) - U_r(E))$ because the policies in $\Pi_E$ achieve zero regret. 
Then Lemma \ref{lm:app_a_6} states that $\Pi_E$ will also be the solution to $\arg\underset{\pi_P\in \Pi}{\max}\ \underset{r\in R_{E,\delta^*}}{\min}\ \left\{\underset{\pi_A\in \Pi}{\max}\ U_r(\pi_A)\right\} - U_r(\pi_P)$. We finish the proof.
\end{proof}

 \noindent{\textbf{Proposition \ref{th:pagar_1_6}.}
If $r^*$ aligns with the task and $\delta \geq \delta^* - (\sup S_{r^*} - \inf S_{r^*})$, the optimal protagonist policy $\pi_P=MinimiaxRegret(R_{E,\delta})$ can succeed in the task.
\begin{proof}
    If $\pi_{r^*}\in MinimiaxRegret(R_{E,\delta})$, then $\pi_{r^*}$ can succeed in the task by definition. 
    Now assume that $\pi_P \neq \pi_{r^*}$. 
    Since $J_{IRL}$ achieves Nash Equilibrium at $r^*$ and $\pi_{r^*}$, for any other reward function $r$ we have $\underset{\pi\in\Pi}{\max}\ U_{r}(\pi)-U_{r}(\pi_{r^*})\leq \delta^* - (U_r(E) - \underset{\pi\in\Pi}{\max}\ U_{r}(\pi)) \leq \delta^* - \delta$.
    We also have $\underset{r'\in R_{E,\delta}}{\max}\ Regret(r', \pi_P)\leq \underset{r'\in R_{E,\delta}}{\max}\ Regret(r', \pi_{r^*})\leq \delta^*-\delta$. 
    Furthermore, $Regre(r^*, \pi_P)\leq \underset{r'\in R_{E,\delta}}{\max}\ Regret(r', \pi_P)$.
    Hence, $Regre(r^*, \pi_P)\leq \delta-\delta^*\leq \sup S_{r^+} - \inf S_{r^+}$.
    In other words, $U_{r^*}(\pi_P)\in S_{r^*}$, indicating $\pi_P$ can succeed in the task.
    The proof is complete.
\end{proof}

 \subsection{Stationary Solutions}\label{subsec:app_a_6}
In this section, we show that $MinimaxRegret$ is convex for $\pi_P$.

\begin{proposition} 
$\underset{r\in R}{\max}\ Regret(\pi_P, r)$ is convex in $\pi_P$. 
\end{proposition}
\begin{proof}
For any $\alpha\in[0, 1]$ and $\pi_{P,1}, \pi_{P,2}$, there exists a $\pi_{P,3}=\alpha \pi_{P,1} + (1-\alpha) \pi_{P,2}$. 
Let $r_1, \pi_{A,1}$ and $r_2, \pi_{A,2}$ be the optimal reward and antagonist policy for $\pi_{P,1}$ and $\pi_{P,2}$
Then $\alpha \cdot (\underset{r\in R}{\max}\ \underset{\pi_A\in\Pi}{\max}\ U_r(\pi_A)- U_r(\pi_{P,1}))+(1-\alpha)\cdot(\underset{r\in R}{\max}\ \underset{\pi_A\in\Pi}{\max}\ U_r(\pi_A)- U_r(\pi_{P,2}))=\alpha (U_{r_1}(\pi_{A,1}) - U_{r_1}(\pi_{P,1})) + (1 - \alpha)(U_{r_2}(\pi_{A,2}) - U_{r_2}(\pi_{P,2}))\geq \alpha (U_{r_3}(\pi_{A,3}) - U_{r_3}(\pi_{P,1})) + (1 -\alpha)(U_{r_2}(\pi_{A,3}) - U_{r_3}(\pi_{P,2}))=U_{r_3}(\pi_{A,3}) - U_{r_3}(\pi_{P,3})$. 
Therefore, $\underset{r\in R}{\max}\ \underset{\pi_A\in\Pi}{\max}\ U_r(\pi_A)- U_r(\pi_P)$ is convex in $\pi_P$.
\end{proof}

 \subsection{Details of Example 1}\label{subsec:app_a_7}
We reiterate the content of the example and Figure \ref{fig:exp_fig0} here for the reader's convenience.

\begin{example}\label{exp_1}
Figure \ref{fig:app_a_1} shows an illustrative example of how PAGAR-based IL mitigates reward misalignment in IRL-based IL.
The task requires that \textit{a policy must visit $s_2$ and $s_6$ with no less than $0.5$ probability within $5$ steps}, i.e. $Prob(s_2|\pi)\geq 0.5 \wedge Prob(s_6|\pi)\geq 0.5$ where $Prob(s|\pi)$ is the probability of $\pi$ generating a trajectory that contains $s$ within the first $5$ steps.
We derive that a successful policy must choose $a_2$ at $s_0$ with a probability within $[\frac{1}{2},\frac{125}{188}]$ as follows.

The trajectories that reach $s_6$ after choosing $a_2$ at $s_0$ include: $(s_0,a_2,s_2,s_6), (s_0,a_2, s_2, s_2, s_6), (s_0, a_2, s_2,s_2,s_2, s_6)$.
The total probability equals $Prob(s_6|\pi; s_0,a_2)=\frac{1}{5} + \frac{1}{5}^2 + \frac{1}{5}^3=\frac{31}{125}$. 
Then the total probability of reaching $s_6$ equals
$Prob(s_6|\pi)=(1-\pi(a_2|s_0)) + \frac{31}{125}\cdot \pi(a_2|s_0)$.
For $Prob(s_6|\pi)$ to be no less than $0.5$, $\pi(a_2|s_0)$ must be no greater than $\frac{125}{188}$.

The reward function hypothesis space is $\{r_\omega|r_\omega(s,a)=\omega \cdot r_1(s,a)+(1-\omega)\cdot r_2(s,a)\}$ where $\omega\in[0, 1]$ is a parameter, $r_1, r_2$ are two features.
Specifically, $r_1(s,a)$ equals $1$ if $s=s_2$ and equals $0$ otherwise,  and $r_2(s,a)$ equals $1$ if $s=s_6$ and equals $0$ otherwise.
Given the demonstrations and the MDP, the maximum negative MaxEnt IRL loss $\delta^*\approx 2.8$ corresponds to the optimal parameter 
 $\omega^*=1$.
 This is computed based on Eq.6 in \cite{maxentirl}. 
The discount factor is $\gamma=0.99$.
When computing the normalization term $Z$ in Eq.4 of \cite{maxentirl}, we only consider the trajectories within $5$ steps. 
In Figure \ref{fig:app_a_1}, we further show how the MaxEnt IRL loss changes with $\omega$. 
 The optimal policy under $r_{\omega^*}$ chooses $a_2$ at $s_0$ with probability $1$, thus failing to accomplish the task. 
The optimal protagonist policy $\pi_P=MinimaxRegret(R_{E,\delta})$ also chooses $a_2$ at $s_0$ with probability $1$ when $\delta$ is close to its maximum $\delta^*$.
However, $\pi_P(a_2|s_2)$ decreases as $\delta$ decreases.
It turns out that for any $\delta < 1.2$ the optimal protagonist policy can succeed in the task.
\end{example}
\begin{figure}[tph!]
\hfill
\includegraphics[width=0.24\linewidth]{figures/fig1.png}
\hfill
\includegraphics[width=0.37\linewidth]{figures/fig3.png}
\hfill
\includegraphics[width=0.37\linewidth]{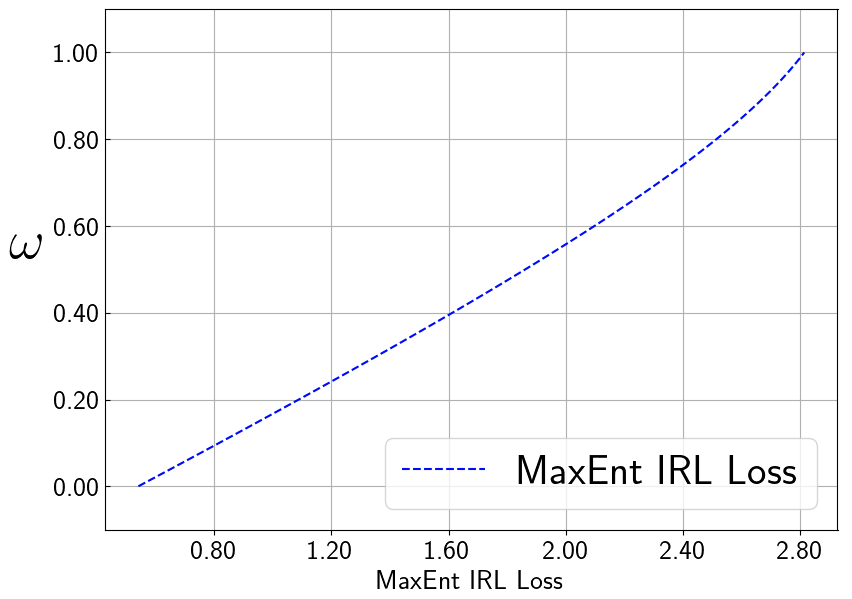}
\hfill
\caption{\textbf{Left:} Consider an MDP where there are two available actions $a_1,a_2$ at initial state $s_0$. 
In other states, actions make no difference: the transition probabilities are either annotated at the transition edges or equal $1$ by default. 
States $s_3$ and $s_6$ are terminal states. 
Expert demonstrations are in $E$.
\textbf{Middle}: x-axis indicates the MaxEnt IRL loss bound $\delta$ for $R_{E,\delta}$ as defined in Section \ref{subsec:app_a_1}. 
The y-axis indicates the probability of the protagonist policy learned via $MinimaxRegret(R_{E,\delta})$ choosing $a_2$ at $s_0$.
The red curve shows how different $\delta$'s lead to different protagonist policies.
The blue dashed curve is for reference, showing the optimal policy under the optimal reward learned via MaxEnt IRL.
\textbf{Right}: The curve shows how the MaxEnt IRL Loss changes with $\omega$.
} \label{fig:app_a_1}
\end{figure}
\section{Approach to Solving MinimaxRegret}\label{sec:app_b}
In this section, we develop a series of theories that lead to two bounds of the Protagonist Antagonist Induced Regret.
By using those bounds, we formulate objective functions for solving Imitation Learning problems with PAGAR.

\subsection{Protagonist Antagonist Induced Regret Bounds}\label{subsec:app_b_1}

Our theories are inspired by the on-policy policy improvement methods in \cite{trpo}. 
The theories in \cite{trpo} are under the setting where entropy regularizer is not considered.
In our implementation, we always consider entropy regularized RL of which the objective is to learn a policy that maximizes $J_{RL}(\pi;r)=U_r(\pi) + \mathcal{H}(\pi)$.   
Also, since we use GAN-based IRL algorithms, the learned reward function $r$ as proved by \cite{airl} is a distribution.
Moreover, it is also proved in \cite{airl} that a policy $\pi$ being optimal under $r$ indicates that $\log \pi\equiv r\equiv  \mathcal{A}_\pi$. 
We omit the proof and let the reader refer to \cite{airl} for details.
Although all our theories are about the relationship between the Protagonist Antagonist Induced Regret and the soft advantage function $\mathcal{A}_\pi$, the equivalence between $\mathcal{A}_\pi$ and $r$ allows us to use the theories to formulate our reward optimization objective functions. 
To start off, we denote the reward function to be optimized as $r$.
Given the intermediate learned reward function $r$, we study the Protagonist Antagonist Induced Regret between two policies $\pi_1$ and $\pi_2$.

\begin{lemma}\label{lm:app_b_1}
Given a reward function $r$ and a pair of policies $\pi_1$  and $\pi_2$, 
 \begin{eqnarray}
 U_r(\pi_1) - U_r(\pi_2)=\underset{\tau\sim\pi_1}{\mathbb{E}}\left[\sum^\infty_{t=0}\gamma^t \mathcal{A}_{\pi_2}(s^{(t)}, a^{(t)})
 \right] + \underset{\tau\sim \pi}{\mathbb{E}}\left[\sum^\infty_{t=0} \gamma^t \mathcal{H}\left(\pi_2(\cdot|s^{(t)})\right)\right]
 \end{eqnarray}
\end{lemma}
\begin{proof}
This proof follows the proof of Lemma 1 in \cite{trpo} where RL is not entropy-regularized. For entropy-regularized RL, since $\mathcal{A}_\pi(s,a^{(t)})=\underset{s'\sim \mathcal{P}(\cdot|s,a^{(t)})}{\mathbb{E}}\left[r(s, a^{(t)}) + \gamma \mathcal{V}_\pi(s') - \mathcal{V}_\pi(s)\right]$, 
\begin{eqnarray}
    &&\underset{\tau \sim \pi_1}{\mathbb{E}}\left[\sum^\infty_{t=0} \gamma^t \mathcal{A}_{\pi_2}(s^{(t)}, a^{(t)})\right]\nonumber\\
    &=&\underset{\tau \sim \pi_1}{\mathbb{E}}\left[\sum^\infty_{t=0} \gamma^t\left(r(s^{(t+1)}, a^{(t+1)}) + \gamma \mathcal{V}_{\pi_2}(s^{(t+1)}) - \mathcal{V}_{\pi_2}(s^{(t)})\right)\right]\nonumber\\
    &=&\underset{\tau \sim \pi_1}{\mathbb{E}}\left[\sum^\infty_{t=0} \gamma^t r(s^{(t)}, a^{(t)}) -\mathcal{V}_{\pi_2}(s^{(0)})\right]\nonumber\\
    &=&\underset{\tau \sim \pi_1}{\mathbb{E}}\left[\sum^\infty_{t=0} \gamma^t r(s^{(t)}, a^{(t)})\right] -\underset{s^{(0)}\sim d_0}{\mathbb{E}}\left[\mathcal{V}_{\pi_2}(s^{(0)})\right]\nonumber\\
    &=&\underset{\tau \sim \pi_1}{\mathbb{E}}\left[\sum^\infty_{t=0} \gamma^t r(s^{(t)}, a^{(t)})\right] -\underset{\tau \sim \pi_2}{\mathbb{E}}\left[\sum^\infty_{t=0} \gamma^t r(s^{(t)}, a^{(t)}) + \mathcal{H}\left(\pi_2(\cdot|s^{(t)})\right)\right]\nonumber\\
    &=& U_r (\pi_1) - U_r(\pi_2) - \underset{\tau\sim \pi_2}{\mathbb{E}}\left[\sum^\infty_{t=0} \gamma^t \mathcal{H}\left(\pi_2(\cdot|s^{(t)})\right)\right]\nonumber\\
    &=& U_r (\pi_1) - U_r(\pi_2) - \mathcal{H}(\pi_2)\nonumber
\end{eqnarray}
\end{proof}

\begin{remark}\label{rm:app_b_rm1}
Lemma \ref{lm:app_b_1} confirms that $\underset{\tau \sim \pi}{\mathbb{E}}\left[\sum^\infty_{t=0} \gamma^t \mathcal{A}_{\pi}(s^{(t)}, a^{(t)})\right] = U_r(\pi) - U_r(\pi) + \mathcal{H}(\pi)=\mathcal{H}(\pi)$.
\end{remark}

We follow \cite{trpo} and denote $\Delta \mathcal{A}(s)=
\underset{a\sim\pi_1(\cdot|s)}{\mathbb{E}}\left[\mathcal{A}_{\pi_2}(s,a)\right] - \underset{a\sim\pi_2(\cdot|s)}{\mathbb{E}}\left[\mathcal{A}_{\pi_2}(s,a)\right] $ as the difference between the expected advantages of following $\pi_2$  after choosing an action respectively by following policy $\pi_1$ and $\pi_2$ at any state $s$. 
 Although the setting of \cite{trpo} differs from ours by having the expected advantage $\underset{a\sim\pi_2(\cdot|s)}{\mathbb{E}}\left[\mathcal{A}_{\pi_2}(s,a)\right]$ equal to $0$ due to the absence of entropy regularization, the following definition and lemmas from \cite{trpo} remain valid in our setting.

\begin{definition}~\cite{trpo}, the protagonist policy $\pi_1$ and the antagonist policy $\pi_2)$ are $\alpha$-coupled if they defines a joint distribution over $(a, \tilde{a})\in \mathbb{A}\times \mathbb{A}$, such that $Prob(a \neq \tilde{a}|s)\leq \alpha$ for all $s$.
\end{definition}

\begin{lemma}~\cite{trpo}
Given that the protagonist policy $\pi_1$ and the antagonist policy $\pi_2$ are $\alpha$-coupled, then for all state $s$,
\begin{eqnarray}
|\Delta \mathcal{A}(s)|\leq  2\alpha \underset{a}{\max}|\mathcal{A}_{\pi_2}(s,a)|
\end{eqnarray}
\end{lemma}
 
\begin{lemma}\label{lm:app_b_2}~\cite{trpo}
Given that the protagonist policy $\pi_1$ and the antagonist policy $\pi_2$ are $\alpha$-coupled, then
\begin{eqnarray}
\left|\underset{s^{(t)}\sim \pi_1}{\mathbb{E}}\left[\Delta \mathcal{A}(s^{(t)})\right] - \underset{s^{(t)}\sim \pi_2}{\mathbb{E}}\left[\Delta \mathcal{A}(s^{(t)})\right]\right|\leq  4\alpha(1- (1-\alpha)^t) \underset{s,a}{\max}|\mathcal{A}_{\pi_2}(s,a)|
\end{eqnarray}
\end{lemma}

\begin{lemma}\label{lm:app_b_3}
Given that the protagonist policy $\pi_1$ and the antagonist policy $\pi_2$ are $\alpha$-coupled, then
\begin{eqnarray}
\underset{\substack{s^{(t)}\sim \pi_1\\a^{(t)}\sim \pi_2}}{\mathbb{E}}\left[\mathcal{A}_{\pi_2}(s^{(t)},a^{(t)})\right] - \underset{\substack{s^{(t)}\sim \pi_2\\a^{(t)}\sim \pi_2}}{\mathbb{E}}\left[ \mathcal{A}_{\pi_2}(s^{(t)},a^{(t)})\right] \leq 2(1 - (1-\alpha)^t)\underset{(s,a)}{\max}|\mathcal{A}_{\pi_2}(s,a)|
\end{eqnarray}
\end{lemma}
 \begin{proof}
    The proof is similar to that of Lemma \ref{lm:app_b_2} in \cite{trpo}. Let $n_t$ be the number of times that $a^{(t')}\sim \pi_1$ does not equal $a^{(t')}\sim \pi_2$ for $t'< t$, i.e., the number of times that $\pi_1$ and $\pi_2$ disagree before timestep $t$. Then for $s^{(t)}\sim \pi_1$, we have the following.
    \begin{eqnarray}
        &&\underset{s^{(t)}\sim \pi_1}{\mathbb{E}}\left[\underset{a^{(t)}\sim \pi_2}{\mathbb{E}}\left[\mathcal{A}_{\pi_2}(s^{(t)},a^{(t)})\right] \right]\nonumber\\
        &=& P(n_t=0)  \underset{\substack{s^{(t)}\sim \pi_1\\ n_t=0}}{\mathbb{E}}\left[\underset{a^{(t)}\sim\pi_2}{\mathbb{E}}\left[\mathcal{A}_{\pi_2}(s^{(t)},a^{(t)})\right] \right] + P(n_t > 0)\underset{\substack{s^{(t)}\sim \pi_1\\n_t > 0}}{\mathbb{E}}\left[\underset{a^{(t)}\sim\pi_2}{\mathbb{E}}\left[\mathcal{A}_{\pi_2}(s^{(t)},a^{(t)})\right] \right]\nonumber
    \end{eqnarray}
    The expectation decomposes similarly for $s^{(t)}\sim\pi_2$.
    \begin{eqnarray}
        &&\underset{\substack{s^{(t)}\sim \pi_2\\a^{(t)}\sim \pi_2}}{\mathbb{E}}\left[ \mathcal{A}_{\pi_2}(s^{(t)},a^{(t)})\right]\nonumber\\
        =&&P(n_t=0)\underset{\substack{s^{(t)}\sim \pi_2\\a^{(t)}\sim \pi_2\\n_t=0}}{\mathbb{E}}\left[ \mathcal{A}_{\pi_2}(s^{(t)},a^{(t)})\right] + P(n_t>0)\underset{\substack{s^{(t)}\sim \pi_2\\a^{(t)}\sim \pi_2\\n_t > 0}}{\mathbb{E}}\left[ \mathcal{A}_{\pi_2}(s^{(t)},a^{(t)})\right]\nonumber
    \end{eqnarray}
    When computing $\underset{s^{(t)}\sim \pi_1}{\mathbb{E}}\left[\underset{a^{(t)}\sim \pi_2}{\mathbb{E}}\left[\mathcal{A}_{\pi_2}(s^{(t)},a^{(t)})\right] \right]-\underset{\substack{s^{(t)}\sim \pi_2\\a^{(t)}\sim \pi_2}}{\mathbb{E}}\left[ \mathcal{A}_{\pi_2}(s^{(t)},a^{(t)})\right]$, the terms with $n_t=0$ cancel each other because $n_t=0$ indicates that $\pi_1$ and $\pi_2$ agreed on all timesteps less than $t$. That leads to the following.
    \begin{eqnarray}
        &&\underset{s^{(t)}\sim \pi_1}{\mathbb{E}}\left[\underset{a^{(t)}\sim \pi_2}{\mathbb{E}}\left[\mathcal{A}_{\pi_2}(s^{(t)},a^{(t)})\right] \right]-\underset{\substack{s^{(t)}\sim \pi_2\\a^{(t)}\sim \pi_2}}{\mathbb{E}}\left[ \mathcal{A}_{\pi_2}(s^{(t)},a^{(t)})\right]\nonumber\\
        =&&P(n_t > 0)\underset{\substack{s^{(t)}\sim \pi_1\\n_t > 0}}{\mathbb{E}}\left[\underset{a^{(t)}\sim \pi_2}{\mathbb{E}}\left[\mathcal{A}_{\pi_2}(s^{(t)},a^{(t)})\right] \right] - P(n_t>0)\underset{\substack{s^{(t)}\sim \pi_2\\a^{(t)}\sim \pi_2\\n_t > 0}}{\mathbb{E}}\left[ \mathcal{A}_{\pi_2}(s^{(t)},a^{(t)})\right] \nonumber
    \end{eqnarray}
    By definition of $\alpha$, the probability of $\pi_1$ and $\pi_2$ agreeing at timestep $t'$ is no less than $1 - \alpha$. Hence, $P(n_t > 0)\leq 1 - (1- \alpha^t)^t$.  
    Hence, we have the following bound.
    
    \begin{eqnarray}
        &&\left|\underset{s^{(t)}\sim \pi_1}{\mathbb{E}}\left[\underset{a^{(t)}\sim \pi_2}{\mathbb{E}}\left[\mathcal{A}_{\pi_2}(s^{(t)},a^{(t)})\right] \right]-\underset{\substack{s^{(t)}\sim \pi_2\\a^{(t)}\sim \pi_2}}{\mathbb{E}}\left[ \mathcal{A}_{\pi_2}(s^{(t)},a^{(t)})\right]\right|\nonumber\\
        =&&\left|P(n_t > 0)\underset{\substack{s^{(t)}\sim \pi_1\\n_t > 0}}{\mathbb{E}}\left[\underset{a^{(t)}\sim \pi_2}{\mathbb{E}}\left[\mathcal{A}_{\pi_2}(s^{(t)},a^{(t)})\right] \right] - P(n_t>0)\underset{\substack{s^{(t)}\sim \pi_2\\a^{(t)}\sim \pi_2\\n_t > 0}}{\mathbb{E}}\left[ \mathcal{A}_{\pi_2}(s^{(t)},a^{(t)})\right]\right| \nonumber\\
        \leq &&P(n_t > 0)\left(\left|\underset{\substack{s^{(t)}\sim \pi_1\\a^{(t)}\sim \pi_2\\n_t\geq 0}}{\mathbb{E}}\left[\mathcal{A}_{\pi_2}(s^{(t)},a^{(t)})\right]\right| + \left|\underset{\substack{s^{(t)}\sim \pi_2\\a^{(t)}\sim \pi_2\\n_t > 0}}{\mathbb{E}}\left[ \mathcal{A}_{\pi_2}(s^{(t)},a^{(t)})\right]\right| \right)\nonumber\\
        \leq&&2(1 - (1-\alpha)^t)\underset{(s,a)}{\max}|\mathcal{A}_{\pi_2}(s,a)|   \label{eq:app_b_1}
    \end{eqnarray}
\end{proof}

The preceding lemmas lead to the proof for Theorem \ref{th:pagar2_1} in the main text.
 
\noindent{\textbf{Theorem  {\ref{th:pagar2_1}.}} Suppose that $\pi_2$ is the optimal policy in terms of entropy regularized RL under $r$. Let $\alpha = \underset{s}{\max}\ D_{TV}(\pi_1(\cdot|s), \pi_2(\cdot|s))$, $\epsilon = \underset{s,a}{\max}\ |\mathcal{A}_{\pi_2}(s,a^{(t)})|$, and $\Delta\mathcal{A}(s)=\underset{a\sim\pi_1}{\mathbb{E}}\left[\mathcal{A}_{\pi_2}(s,a)\right] - \underset{a\sim\pi_2}{\mathbb{E}}\left[\mathcal{A}_{\pi_2}(s,a)\right]$. For any policy $\pi_1$, the following bounds hold. 
\begin{eqnarray}
    \left|U_r(\pi_1) -U_r(\pi_2) - \sum^\infty_{t=0}\gamma^t\underset{s^{(t)}\sim\pi_1}{\mathbb{E}}\left[\Delta\mathcal{A}(s^{(t)})\right]\right|&\leq&\frac{2\alpha\gamma\epsilon}{(1-\gamma)^2}\label{eq:app_b_8}\\
   \left|U_r(\pi_1)-U_r(\pi_2) - \sum^\infty_{t=0}\gamma^t\underset{s^{(t)}\sim\pi_2}{\mathbb{E}}\left[\Delta\mathcal{A}(s^{(t)})\right]\right|&\leq& \frac{2\alpha\gamma(2\alpha+1)\epsilon}{(1-\gamma)^2} \label{eq:app_b_9}
\end{eqnarray}
}
\begin{proof}
We first leverage Lemma \ref{lm:app_b_1} to derive Eq.\ref{eq:app_b_2}.
Note that  since $\pi_2$ is optimal under $r$, Remark \ref{rm:app_b_rm1} confirmed that $\mathcal{H}(\pi_2)=-\sum^\infty_{t=0}\gamma^t\underset{s^{(t)}\sim\pi_2}{\mathbb{E}}\left[\underset{a^{(t)}\sim\pi_2}{\mathbb{E}}\left[\mathcal{A}_{\pi_2}(s^{(t)},a^{(t)})\right]\right]$.
\begin{eqnarray}
&&U_r(\pi_1) -U_r(\pi_2)\nonumber\\
&=& \left(U_r(\pi_1) -U_r(\pi_2) - \mathcal{H}(\pi_2)\right) + \mathcal{H}(\pi_2)\nonumber\\
&=& \underset{\tau\sim\pi_1}{\mathbb{E}}\left[\sum^\infty_{t=0}\gamma^t \mathcal{A}_{\pi_2}(s^{(t)}, a^{(t)})
 \right] + \mathcal{H}(\pi_2)\nonumber\\
&=& \underset{\tau\sim\pi_1}{\mathbb{E}}\left[\sum^\infty_{t=0}\gamma^t \mathcal{A}_{\pi_2}(s^{(t)}, a^{(t)})
 \right]-\sum^\infty_{t=0}\gamma^t\underset{s^{(t)}\sim\pi_2}{\mathbb{E}}\left[\underset{a^{(t)}\sim\pi_2}{\mathbb{E}}\left[\mathcal{A}_{\pi_2}(s^{(t)},a^{(t)})\right]\right]\nonumber\\
&=&\sum^\infty_{t=0}\gamma^t\underset{s^{(t)}\sim\pi_1}{\mathbb{E}}\left[\underset{a^{(t)}\sim\pi_1}{\mathbb{E}}\left[\mathcal{A}_{\pi_2}(s^{(t)},a^{(t)})\right] - \underset{a^{(t)}\sim\pi_2}{\mathbb{E}}\left[\mathcal{A}_{\pi_2}(s^{(t)},a^{(t)})\right] \right] + \nonumber\\
&& \sum^\infty_{t=0}\gamma^t\left(\underset{s^{(t)}\sim\pi_1}{\mathbb{E}}\left[\underset{a^{(t)}\sim\pi_2}{\mathbb{E}}\left[\mathcal{A}_{\pi_2}(s^{(t)},a^{(t)})\right]\right]-\underset{s^{(t)}\sim\pi_2}{\mathbb{E}}\left[\underset{a^{(t)}\sim\pi_2}{\mathbb{E}}\left[\mathcal{A}_{\pi_2}(s^{(t)},a^{(t)})\right]\right] \right)\nonumber\\
&=&\sum^\infty_{t=0}\gamma^t\underset{s^{(t)}\sim\pi_1}{\mathbb{E}}\left[\Delta\mathcal{A}(s^{(t)})\right] +\nonumber\\
&& \sum^\infty_{t=0}\gamma^t\left(\underset{s^{(t)}\sim\pi_1}{\mathbb{E}}\left[\underset{a^{(t)}\sim\pi_2}{\mathbb{E}}\left[\mathcal{A}_{\pi_2}(s^{(t)},a^{(t)})\right]\right]-\underset{s^{(t)}\sim\pi_2}{\mathbb{E}}\left[\underset{a^{(t)}\sim\pi_2}{\mathbb{E}}\left[\mathcal{A}_{\pi_2}(s^{(t)},a^{(t)})\right]\right] \right)\label{eq:app_b_2}
\end{eqnarray}

We switch terms between Eq.\ref{eq:app_b_2} and $U_r(\pi_1)-U_r(\pi_2)$, then use Lemma \ref{lm:app_b_3} to derive Eq.\ref{eq:app_b_7}.
\begin{eqnarray}
&&\left|U_r(\pi_1) -U_r(\pi_2) - \sum^\infty_{t=0}\gamma^t\underset{s^{(t)}\sim\pi_1}{\mathbb{E}}\left[\Delta\mathcal{A}(s^{(t)})\right]\right|\nonumber\\
&=&\left|\sum^\infty_{t=0}\gamma^t\left(\underset{s^{(t)}\sim\pi_1}{\mathbb{E}}\left[\underset{a^{(t)}\sim\pi_2}{\mathbb{E}}\left[\mathcal{A}_{\pi_2}(s^{(t)},a^{(t)})\right]\right]-\underset{s^{(t)}\sim\pi_2}{\mathbb{E}}\left[\underset{a^{(t)}\sim\pi_2}{\mathbb{E}}\left[\mathcal{A}_{\pi_2}(s^{(t)},a^{(t)})\right]\right] \right)\right|\nonumber\\
&\leq & \sum^\infty_{t=0}\gamma^t \cdot 2\underset{(s,a)}{\max}|\mathcal{A}_{\pi_2}(s,a)|\cdot (1 - (1-\alpha)^t)\leq \frac{2\alpha\gamma\underset{(s,a)}{\max}|\mathcal{A}_{\pi_2}(s,a)|}{(1-\gamma)^2}\label{eq:app_b_7}
\end{eqnarray}
Alternatively, we can expand $U_r(\pi_2)-U_r(\pi_1)$ into Eq.\ref{eq:app_b_3}.
During the process, $\mathcal{H}(\pi_2)$ is converted into $-\sum^\infty_{t=0}\gamma^t\underset{s^{(t)}\sim\pi_2}{\mathbb{E}}\left[\underset{a^{(t)}\sim\pi_2}{\mathbb{E}}\left[\mathcal{A}_{\pi_2}(s^{(t)},a^{(t)})\right]\right]$.
\begin{eqnarray}
&&U_r(\pi_1) -U_r(\pi_2)\nonumber\\
&=& \left(U_r(\pi_1) -U_r(\pi_2) - \mathcal{H}(\pi_2)\right)+\mathcal{H}(\pi_2)\nonumber\\
&=&\underset{\tau\sim\pi_1}{\mathbb{E}}\left[\sum^\infty_{t=0}\gamma^t \mathcal{A}_{\pi_2}(s^{(t)}, a^{(t)})
 \right]+\mathcal{H}(\pi_2)\nonumber\\
&=&\sum^\infty_{t=0}\gamma^t\underset{s^{(t)}\sim\pi_1}{\mathbb{E}}\left[\underset{a^{(t)}\sim\pi_1}{\mathbb{E}} \left[\mathcal{A}_{\pi_2}(s^{(t)}, a^{(t)})
 \right]\right]+\mathcal{H}(\pi_2)\nonumber\\
 &=&\sum^\infty_{t=0}\gamma^t\underset{s^{(t)}\sim\pi_1}{\mathbb{E}}\left[\Delta A(s^{(t)}) +\underset{a^{(t)}\sim\pi_2}{\mathbb{E}} \left[\mathcal{A}_{\pi_2}(s^{(t)}, a^{(t)})
 \right]\right]+\mathcal{H}(\pi_2)\nonumber\\
&=& \sum^\infty_{t=0}\gamma^t\underset{s^{(t)}\sim\pi_2}{\mathbb{E}}\left[
 \underset{a^{(t)}\sim\pi_1}{\mathbb{E}}\left[\mathcal{A}_{\pi_2}(s^{(t)},a^{(t)})\right]-
 \underset{a^{(t)}\sim\pi_2}{\mathbb{E}}\left[\mathcal{A}_{\pi_2}(s^{(t)}, a^{(t)})\right] - \Delta \mathcal{A}(s^{(t)})\right] + \nonumber\\
&&\underset{s^{(t)}\sim \pi_1}{\mathbb{E}}\left[ \Delta \mathcal{A}(s^{(t)}) +  \underset{a^{(t)}\sim\pi_2}{\mathbb{E}} \left[\mathcal{A}_{\pi_2}(s^{(t)}, a^{(t)})
 \right]\right] -  \underset{\substack{s^{(t)}\sim \pi_2\\a^{(t)}\sim\pi_2}}{\mathbb{E}} \left[\mathcal{A}_{\pi_2}(s^{(t)}, a^{(t)})
 \right]\nonumber\\
 &=&  \sum^\infty_{t=0}\gamma^t\left(\underset{s^{(t)}\sim \pi_1}{\mathbb{E}} \left[\underset{a^{(t)}\sim\pi_2}{\mathbb{E}}\left[\mathcal{A}_{\pi_2}(s^{(t)}, a^{(t)} )
 \right]\right]- 2 \underset{s^{(t)}\sim\pi_2}{\mathbb{E}}\left[  \underset{a^{(t)}\sim\pi_2}{\mathbb{E}}\left[\mathcal{A}_{\pi_2}(s^{(t)}, a^{(t)})\right]\right]\right)+ \nonumber\\
  &&\sum^\infty_{t=0}\gamma^t\left(\underset{s^{(t)}\sim\pi_2}{\mathbb{E}}\left[\underset{a^{(t)}\sim\pi_1}{\mathbb{E}}\left[\mathcal{A}_{\pi_2}(s^{(t)},a^{(t)})\right]  \right] - (\underset{s^{(t)}\sim \pi_2}{\mathbb{E}}\left[\Delta \mathcal{A}(s^{(t)})\right] - \underset{s^{(t)}\sim \pi_1}{\mathbb{E}}\left[\Delta \mathcal{A}(s^{(t)})\right])\right)\nonumber\\\label{eq:app_b_3}
 \end{eqnarray}
We switch terms between Eq.\ref{eq:app_b_3} and $U_r(\pi_1)-U_r(\pi_2)$, then base on Lemma \ref{lm:app_b_2} and \ref{lm:app_b_3} to derive the inequality in Eq.\ref{eq:app_b_4}.
\begin{eqnarray}
&&\left|U_r(\pi_1)-U_r(\pi_2) - \sum^\infty_{t=0}\gamma^t\underset{s^{(t)}\sim\pi_2}{\mathbb{E}}\left[\Delta\mathcal{A}_{\pi}(s^{(t)},a^{(t)})\right]\right|\nonumber\\
&=&\biggl|U_r(\pi_1) -U_r(\pi_2) - \nonumber\\
&&\qquad  \sum^\infty_{t=0}\gamma^t\left(\underset{s^{(t)}\sim\pi_2}{\mathbb{E}}\left[\underset{a^{(t)}\sim\pi_1}{\mathbb{E}}\left[\mathcal{A}_{\pi_2}(s^{(t)},a^{(t)})\right]  \right] -\underset{s^{(t)}\sim\pi_2}{\mathbb{E}}\left[  \underset{a^{(t)}\sim\pi_2}{\mathbb{E}}\left[\mathcal{A}_{\pi_2}(s^{(t)}, a^{(t)})\right]\right]\right)\biggr|\nonumber\\
&=& \biggl|\sum^\infty_{t=0}\gamma^t \left(\underset{s^{(t)}\sim \pi_2}{\mathbb{E}}\left[\Delta \mathcal{A}(s^{(t)})\right] - \underset{s^{(t)}\sim \pi_1}{\mathbb{E}}\left[\Delta \mathcal{A}(s^{(t)})\right]\right)-\nonumber\\
 &&\qquad \sum^\infty_{t=0}\gamma^t\left(\underset{s^{(t)}\sim \pi_1}{\mathbb{E}} \left[\underset{a^{(t)}\sim\pi_2}{\mathbb{E}}\left[\mathcal{A}_{\pi_2}(s^{(t)}, a^{(t)} )
 \right]\right]- \underset{s^{(t)}\sim\pi_2}{\mathbb{E}}\left[  \underset{a^{(t)}\sim\pi_2}{\mathbb{E}}\left[\mathcal{A}_{\pi_2}(s^{(t)}, a^{(t)})\right]\right]\right)  \biggr| \nonumber\\
&\leq&   \left|\sum^\infty_{t=0}\gamma^t \left(\underset{s^{(t)}\sim \pi_2}{\mathbb{E}}\left[\Delta \mathcal{A}(s^{(t)})\right] - \underset{s^{(t)}\sim \pi_1}{\mathbb{E}}\left[\Delta \mathcal{A}(s^{(t)})\right]\right)\right| + \nonumber\\
 &&\qquad \left|\sum^\infty_{t=0}\gamma^t\left(\underset{s^{(t)}\sim \pi_1}{\mathbb{E}} \left[\underset{a^{(t)}\sim\pi_2}{\mathbb{E}}\left[\mathcal{A}_{\pi_2}(s^{(t)}, a^{(t)} )
 \right]\right]- \underset{s^{(t)}\sim\pi_2}{\mathbb{E}}\left[  \underset{a^{(t)}\sim\pi_2}{\mathbb{E}}\left[\mathcal{A}_{\pi_2}(s^{(t)}, a^{(t)})\right]\right]\right)  \right| \nonumber\\
 &\leq & \sum^\infty_{t=0}\gamma^t \left((1- (1-\alpha)^t) (4\alpha\underset{s,a}{\max}|\mathcal{A}_{\pi_2}(s,a)| + 2\underset{(s,a)}{\max}|\mathcal{A}_{\pi_2}(s,a)|)\right)\nonumber\\
&\leq &\frac{2\alpha\gamma(2\alpha+1)\underset{s,a}{\max}|\mathcal{A}_{\pi_2}(s,a)|}{(1-\gamma)^2} \label{eq:app_b_4}
\end{eqnarray}

It is stated in \cite{trpo} that $\underset{s}{\max}\ D_{TV}(\pi_2(\cdot|s), \pi_1(\cdot|s))\leq \alpha$.
Hence, by letting $\alpha:= \underset{s}{\max}\ D_{TV}(\pi_2(\cdot|s), \pi_1(\cdot|s))$, Eq.\ref{eq:app_b_2} and \ref{eq:app_b_4} still hold.
Then,  we have proved Theorem \ref{th:pagar2_1}.
\end{proof}

\subsection{Objective Functions of Reward Optimization}\label{subsec:app_b_2}
To derive $J_{R,1}$ and $J_{R,2}$, we let $\pi_1=\pi_P$ and $\pi_2=\pi_A$. 
Then based on Eq.\ref{eq:app_b_8} and \ref{eq:app_b_9} we derive the following upper-bounds of $U_r(\pi_P)-U_r(\pi_A)$.
\begin{eqnarray}
  U_r(\pi_P) -U_r(\pi_A)&\leq& \sum^\infty_{t=0}\gamma^t\underset{s^{(t)}\sim\pi_P}{\mathbb{E}}\left[\Delta\mathcal{A}(s^{(t)})\right] + \frac{2\alpha\gamma(2\alpha+1)\epsilon}{(1-\gamma)^2}\label{eq:app_b_10}\\
  U_r(\pi_P) - U_r(\pi_A)&\geq& \sum^\infty_{t=0}\gamma^t\underset{s^{(t)}\sim\pi_A}{\mathbb{E}}\left[\Delta\mathcal{A}(s^{(t)})\right] - \frac{2\alpha\gamma\epsilon}{(1-\gamma)^2}\label{eq:app_b_11}
\end{eqnarray}
By our assumption that $\pi_A$ is optimal under $r$, we have $\mathcal{A}_{\pi_A}\equiv r$~\cite{airl}.
This equivalence enables us to replace $\mathcal{A}_{\pi_A}$'s in $\Delta\mathcal{A}$ with $r$. 
As for the $\frac{2\alpha\gamma(2\alpha+1)\epsilon}{(1-\gamma)^2}$ and $\frac{2\alpha\gamma\epsilon}{(1-\gamma)^2}$ terms, 
since the objective is to maximize $U_r(\pi_A)-U_r(\pi_B)$, we heuristically estimate the $\epsilon$ in Eq.\ref{eq:app_b_10} by using the samples from $\pi_P$ and the $\epsilon$ in Eq.\ref{eq:app_b_11} by using the samples from $\pi_A$.
As a result we have the objective functions defined as Eq.\ref{eq:app_b_12} and \ref{eq:app_b_13} where $\delta_1(s,a)=\frac{\pi_P(a^{(t)}|s^{(t)})}{\pi_A(a^{(t)}|s^{(t)})}$ and $\delta_2=\frac{\pi_A(a^{(t)}|s^{(t)})}{\pi_P(a^{(t)}|s^{(t)})}$ are the importance sampling probability ratio derived from the definition of $\Delta\mathcal{A}$; $C_1 \propto - \frac{\gamma \hat{\alpha}}{(1-\gamma)}$ and $C_2\propto\frac{\gamma \hat{\alpha}}{(1-\gamma)}$ where $\hat{\alpha}$ is either an estimated maximal KL-divergence between $\pi_A$ and $\pi_B$ since $D_{KL}\geq D_{TV}^2$ according to \cite{trpo}, or an estimated maximal $D_{TV}^2$ depending on whether the reward function is Gaussian or Categorical. 
We also note that for finite horizon tasks, we compute the average rewards instead of the discounted accumulated rewards in Eq.\ref{eq:app_b_13} and \ref{eq:app_b_12}.
\begin{eqnarray}
    &J_{R,1}(r;\pi_P, \pi_A):=  \underset{\tau\sim \pi_A}{\mathbb{E}}\left[\sum^\infty_{t=0}\gamma^t\left(\delta_1(s^{(t)},a^{(t)})- 1\right)\cdot r(s^{(t)},a^{(t)})\right] +C_1 \underset{(s,a)\sim \pi_A}{\max}|r(s,a)|&\ \   \label{eq:app_b_12}\\
    &J_{R,2}(r; \pi_P, \pi_A) := \underset{\tau\sim \pi_P}{\mathbb{E}}\left[\sum^\infty_{t=0}\gamma^t\left(1 - \delta_2(s^{(t)}, a^{(t)})\right) \cdot r(s^{(t)},a^{(t)})\right] + C_2 \underset{(s,a)\sim \pi_P}{\max}|r(s,a)|&\ \ \label{eq:app_b_13}
\end{eqnarray}

Beside $J_{R, 1}, J_{R, 2}$, we additionally use two more objective functions based on the derived bounds. W $J_{R,r}(r;\pi_A, \pi_P)$.
By denoting the optimal policy under $r$ as $\pi^*$, $\alpha^*=\underset{s\in\mathbb{S}}{\max}\ D_{TV}(\pi^*(\cdot|s), \pi_A(\cdot|s)$, $\epsilon^*=\underset{(s,a^{(t)})}{\max}|\mathcal{A}_{\pi^*}(s,a^{(t)})|$, and $\Delta\mathcal{A}_A^*(s)=\underset{a\sim\pi_A}{\mathbb{E}}\left[\mathcal{A}_{\pi^*}(s,a)\right] - \underset{a\sim\pi^*}{\mathbb{E}}\left[\mathcal{A}_{\pi^*}(s,a)\right]$,  we have the following.

\begin{eqnarray}
&&  U_r(\pi_P) -   U_r(\pi^*)\nonumber\\
&=& U_r(\pi_P) - U_r(\pi_A) + U_r(\pi_A) - U_r(\pi^*) \nonumber\\
&\leq& U_r(\pi_P)-U_r(\pi_A) + \sum^\infty_{t=0}\gamma^t\underset{s^{(t)}\sim\pi_A}{\mathbb{E}}\left[\Delta\mathcal{A}_A^*(s^{(t)})\right] +\frac{2\alpha^*\gamma\epsilon^*}{(1-\gamma)^2}\nonumber\\
&=& U_r(\pi_P)-\sum^\infty_{t=0}\gamma^t\underset{s^{(t)}\sim\pi_A}{\mathbb{E}}\left[\underset{a^{(t)}\sim\pi_A}{\mathbb{E}}\left[r(s^{(t)},a^{(t)})\right]\right] + \nonumber\\ 
&&\qquad \sum^\infty_{t=0}\gamma^t\underset{s^{(t)}\sim\pi_A}{\mathbb{E}}\left[\underset{a^{(t)}\sim\pi_A}{\mathbb{E}}\left[\mathcal{A}_{\pi^*}(s^{(t)},a^{(t)})\right] - \underset{a^{(t)}\sim\pi^*}{\mathbb{E}}\left[\mathcal{A}_{\pi^*}(s^{(t)},a^{(t)})\right]\right] + \frac{2\alpha^*\gamma\epsilon^*}{(1-\gamma)^2}\nonumber\\
&=& U_r(\pi_P) -\sum^\infty_{t=0}\gamma^t\underset{s^{(t)}\sim\pi_A}{\mathbb{E}}\left[  \underset{a^{(t)}\sim\pi^*}{\mathbb{E}}\left[\mathcal{A}_{\pi^*}(s^{(t)},a^{(t)})\right]\right]+ \frac{2\alpha^*\gamma\epsilon^*}{(1-\gamma)^2}\nonumber\\
&=&\underset{\tau\sim \pi_P}{\mathbb{E}}\left[\sum^\infty_{t=0}\gamma^t r(s^{(t)},a^{(t)})\right] -\underset{\tau\sim \pi_A}{\mathbb{E}}\left[\sum^\infty_{t=0}\gamma^t\frac{\exp(r(s^{(t)},a^{(t)}))}{\pi_A(a^{(t)}|s^{(t)})}r(s^{(t)},a^{(t)})\right]+\frac{2\alpha^*\gamma\epsilon^*}{(1-\gamma)^2}\qquad
\label{eq:app_b_5}
\end{eqnarray}

Let $\delta_3=\frac{\exp(r(s^{(t)},a^{(t)}))}{\pi_A(a^{(t)}|s^{(t)})}$ be the importance sampling probability ratio. 
It is suggested in \cite{ppo} that instead of directly optimizing the objective function Eq.\ref{eq:app_b_5}, optimizing a surrogate objective function as in Eq.\ref{eq:app_b_6}, which is an upper-bound of Eq.\ref{eq:app_b_5}, with some small $\delta\in(0, 1)$ can be much less expensive and still effective.

\begin{eqnarray}
    &J_{R,3}(r;\pi_P, \pi_A):=\underset{\tau\sim \pi_P}{\mathbb{E}}\left[\sum^\infty_{t=0}\gamma^t r(s^{(t)},a^{(t)})\right] -\qquad\qquad\qquad\qquad\qquad\qquad\qquad\qquad\quad & \nonumber\\
    &\underset{\tau\sim \pi_A}{\mathbb{E}}\left[\sum^\infty_{t=0}\gamma^t \min\left(\delta_3 \cdot r(s^{(t)},a^{(t)}), clip(\delta_3, 1-\delta, 1 + \delta)\cdot r(s^{(t)},a^{(t)})\right) \right] & \label{eq:app_b_6}
\end{eqnarray}

Alternatively, we let $\Delta\mathcal{A}_P^*(s)=\underset{a\sim\pi_P}{\mathbb{E}}\left[\mathcal{A}_{\pi^*}(s,a)\right] - \underset{a\sim\pi^*}{\mathbb{E}}\left[\mathcal{A}_{\pi^*}(s,a)\right]$.
The according to Eq.\ref{eq:app_b_10}, we have the following.
\begin{eqnarray}
  &&U_r(\pi_P) -U_r(\pi^*)\nonumber\\
  &\leq& \sum^\infty_{t=0}\gamma^t\underset{s^{(t)}\sim\pi_P}{\mathbb{E}}\left[\Delta\mathcal{A}_P^*(s^{(t)})\right] + \frac{2\alpha^*\gamma(2\alpha^*+1)\epsilon^*}{(1-\gamma)^2}\nonumber\\
  &=& \sum^\infty_{t=0}\gamma^t\underset{s^{(t)}\sim\pi_P}{\mathbb{E}}\left[\underset{a^{(t)}\sim\pi_P}{\mathbb{E}}\left[\mathcal{A}_{\pi^*}(s^{(t)},a^{(t)})\right] - \underset{a^{(t)}\sim\pi^*}{\mathbb{E}}\left[\mathcal{A}_{\pi^*}(s^{(t)},a)^{(t)}\right]\right] + \frac{2\alpha^*\gamma(2\alpha^*+1)\epsilon^*}{(1-\gamma)^2}\nonumber\\\label{app_b_7}
\end{eqnarray}

Then a new objective function $J_{R,4}$ is formulated in Eq.\ref{app_b_8} where $\delta_4 = \frac{\exp(r(s^{(t)},a^{(t)}))}{\pi_P(a^{(t)}|s^{(t)})}$.
\begin{eqnarray}
    &J_{R,4}(r;\pi_P, \pi_A):=\underset{\tau\sim \pi_P}{\mathbb{E}}\left[\sum^\infty_{t=0}\gamma^t r(s^{(t)},a^{(t)})\right] -\qquad\qquad\qquad\qquad\qquad\qquad\qquad\qquad\quad & \nonumber\\
    &\underset{\tau\sim \pi_P}{\mathbb{E}}\left[\sum^\infty_{t=0}\gamma^t \min\left(\delta_4 \cdot r(s^{(t)},a^{(t)}), clip(\delta_4, 1-\delta, 1 + \delta)\cdot r(s^{(t)},a^{(t)})\right) \right] & \label{app_b_8}
\end{eqnarray}

\subsection{Incorporating IRL Algorithms}\label{subsec:app_b_3}

In our implementation, we combine PAGAR with GAIL and VAIL, respectively. 
When PAGAR is combined with GAIL, the meta-algorithm Algorithm \ref{alg:pagar2_1} becomes Algorithm \ref{alg:app_b_1}.
When PAGAR is combined with VAIL, it becomes Algorithm \ref{alg:app_b_2}.
Both of the two algorithms are GAN-based IRL, indicating that both algorithms use Eq.\ref{eq:prelm_1} as the IRL objective function. 
In our implementation, we use a neural network to approximate $D$, the discriminator in Eq.\ref{eq:prelm_1}.
To get the reward function $r$, we follow \cite{airl} and denote $r(s,a) = \log \left(\frac{\pi_A(a|s)}{D(s,a)} - \pi_A(a|s)\right)$ as mentioned in Section \ref{sec:intro}.
Hence, the only difference between Algorithm \ref{alg:app_b_1} and Algorithm \ref{alg:pagar2_1} is in the representation of the reward function.
Regarding VAIL, since it additionally learns a  representation for the state-action pairs, a bottleneck constraint $J_{IC}(D)\leq i_c$ is added where the bottleneck $J_{IC}$ is estimated from policy roll-outs.
VAIL introduces a Lagrangian parameter $\beta$ to integrate $J_{IC}(D) - i_c$ in the objective function. 
As a result its objective function becomes $J_{IRL}( r) + \beta\cdot ({J_{IC}(D)} - i_c)$.
VAIL not only learns the policy and the discriminator but also optimizes $\beta$. 
In our case, we utilize the samples from both protagonist and antagonist policies to optimize $\beta$ as in line 10, where we follow \cite{vail} by using projected gradient descent with a step size $\delta$

\begin{algorithm}[tb]
\caption{GAIL w/ PAGAR}
\label{alg:app_b_1}
\textbf{Input}: Expert demonstration $E$, discriminator loss bound $\delta$, initial protagonist policy $\pi_P$, antagonist policy $\pi_A$, discriminator $D$ (representing $r(s,a) = \log \left(\frac{\pi_A(a|s)}{D(s,a)} - \pi_A(a|s)\right)$), Lagrangian parameter $\lambda$, iteration number $i=0$, maximum iteration number $N$ \\
\textbf{Output}: $\pi_P$ 
\begin{algorithmic}[1] 
\WHILE {iteration number $i<N$}
\STATE Sample trajectory sets $\mathbb{D}_A\sim \pi_A$ and $\mathbb{D}_P \sim \pi_P$
\STATE Estimate $J_{RL}(\pi_A;r)$ with $\mathbb{D}_A$
\STATE Optimize $\pi_A$ to maximize $J_{RL}(\pi_A;r)$.
\STATE Estimate $J_{RL}(\pi_P;r)$ with $\mathbb{D}_P$; $J_{\pi_A}(\pi_P; \pi_A, r)$ with $\mathbb{D}_P$ and $\mathbb{D}_A$;
\STATE Optimize $\pi_P$ to maximize $J_{RL}(\pi_P; r) + J_{\pi_A}(\pi_P; \pi_A, r)$.
\STATE Estimate $J_{PAGAR}(r; \pi_P, \pi_A)$ with $\mathbb{D}_P$ and $\mathbb{D}_A$
\STATE Estimate $J_{IRL}(\pi_A; r)$ with $\mathbb{D}_A$ and $E$ by following the IRL algorithm
\STATE Optimize $D$ to minimize $J_{PAGAR}(r; \pi_P, \pi_A) + \lambda \cdot max(J_{IRL}(r) + \delta, 0)$
\ENDWHILE
\STATE \textbf{return} $\pi_P$ 
\end{algorithmic}
\end{algorithm}

\begin{algorithm}[tb]
\caption{VAIL w/ PAGAR}
\label{alg:app_b_2}
\textbf{Input}: Expert demonstration $E$, discriminator loss bound $\delta$, initial protagonist policy $\pi_P$, antagonist policy $\pi_A$, discriminator $D$ (representing $r(s,a) = \log \left(\frac{\pi_A(a|s)}{D(s,a)} - \pi_A(a|s)\right)$), Lagrangian parameter $\lambda$ for PAGAR, iteration number $i=0$, maximum iteration number $N$, Lagrangian parameter $\beta$ for bottleneck constraint, bounds on the bottleneck penalty $i_c$, learning rate $\mu$.  \\
\textbf{Output}: $\pi_P$ 
\begin{algorithmic}[1] 
\WHILE {iteration number $i<N$}
\STATE Sample trajectory sets $\mathbb{D}_A\sim \pi_A$ and $\mathbb{D}_P \sim \pi_P$
\STATE Estimate $J_{RL}(\pi_A;r)$ with $\mathbb{D}_A$
\STATE Optimize $\pi_A$ to maximize $J_{RL}(\pi_A;r)$.
\STATE Estimate $J_{RL}(\pi_P;r)$ with $\mathbb{D}_P$; $J_{\pi_A}(\pi_P; \pi_A, r)$ with $\mathbb{D}_P$ and $\mathbb{D}_A$;
\STATE Optimize $\pi_P$ to maximize $J_{RL}(\pi_P; r) + J_{\pi_A}(\pi_P; \pi_A, r)$.
\STATE Estimate $J_{PAGAR}(r; \pi_P, \pi_A)$ with $\mathbb{D}_P$ and $\mathbb{D}_A$
\STATE Estimate $J_{IRL}(\pi_A; r)$ with $\mathbb{D}_A$ and $E$ by following the IRL algorithm
\STATE Estimate $J_{IC}(D)$ with $\mathbb{D}_A, \mathbb{D}_P$ and $E$
\STATE Optimize $D$ to minimize $J_{PAGAR}(r; \pi_P, \pi_A) + \lambda \cdot max(J_{IRL}(r)+\delta, 0) + \beta\cdot J_{IC}(D)$
\STATE Update $\beta:= \max \left (0, \beta - \mu\cdot (\frac{J_{IC}(D)}{3} - i_c)\right)$
\ENDWHILE
\STATE \textbf{return} $\pi_P$ 
\end{algorithmic}
\end{algorithm}

In our implementation, depending on the difficulty of the benchmarks, we choose to maintain $\lambda$ as a constant or update $\lambda$ with the IRL loss $J_{IRL}( r)$ in most of the continuous control tasks. 
In \textit{HalfCheetah-v2} and all the maze navigation tasks, we update $\lambda$ by introducing a hyperparameter $\mu$.  
As described in the maintext, we treat $\delta$ as the target IRL loss of $J_{IRL}( r)$, i.e., $J_{IRL}( r)\leq \delta$.
In all the maze navigation tasks, we initialize $\lambda$ with some constant $\lambda_0$ and update $\lambda$ by $\lambda:= \lambda \cdot \exp(\mu\cdot(J_{IRL}( r) - \delta))$ after every iteration. 
In \textit{HalfCheetah-v2}, we update $\lambda$ by $\lambda:= max(\lambda_0, \lambda \cdot \exp(\mu\cdot(J_{IRL}( r) - \delta)))$ to avoid $\lambda$ being too small.
Besides, we use PPO~\cite{ppo} to train all policies in Algorithm \ref{alg:app_b_1} and \ref{alg:app_b_2}. 
\section{Experiment Details}\label{sec:app_c}
This section presents some details of the experiments and additional results.
\subsection{Experimental Details}\label{subsec:app_c_1}
\noindent\textbf{Network Architectures}. Our algorithm involves a protagonist policy $\pi_P$,  and an antagonist policy $\pi_A$. 
In our implementation, the two policies have the same structures.
Each structure contains two neural networks, an actor network, and a critic network. 
When associated with GAN-based IRL, we use a discriminator $D$ to represent the reward function as mentioned in Appendix \ref{subsec:app_b_3}. 
\begin{itemize}
    \item \textbf{Protagonist and Antagonist policies}. We prepare two versions of actor-critic networks, a fully connected network (FCN) version, and a CNN version, respectively, for the Mujoco and Mini-Grid benchmarks. 
    The FCN version, the actor and critic networks have $3$ layers.
    Each hidden layer has $100$ neurons and a $tanh$ activation function. 
    The output layer output the mean and standard deviation of the actions.
    In the CNN version, the actor and critic networks share $3$ convolutional layers, each having $5$, $2$, $2$ filters, $2\times2$ kernel size, and $ReLU$ activation function. 
    Then $2$ FCNs are used to simulate the actor and critic networks.
    The FCNs have one hidden layer, of which the sizes are $64$.
    
    \item \textbf{Discriminator $D$ for PAGAR-based GAIL in Algorithm \ref{alg:app_b_1}}. 
    We prepare two versions of discriminator networks, an FCN version and a CNN version, respectively, for the Mujoco and Mini-Grid benchmarks. 
    The FCN version has $3$ linear layers.
    Each hidden layer has $100$ neurons and a $tanh$ activation function. 
    The output layer uses the $Sigmoid$ function to output the confidence.
    In the CNN version, the actor and critic networks share $3$ convolutional layers, each having $5$, $2$, $2$ filters, $2\times2$ kernel size, and $ReLU$ activation function. 
    The last convolutional layer is concatenated with an FCN with one hidden layer with $64$ neurons and $tanh$ activation function. 
    The output layer uses the $Sigmoid$ function as the activation function.
    \item \textbf{Discriminator $D$ for PAGAR-based VAIL in Algorithm \ref{alg:app_b_2}}. 
    We prepare two versions of discriminator networks, an FCN version and a CNN version, respectively, for the Mujoco and Mini-Grid benchmarks. 
    The FCN version uses $3$ linear layers to generate the mean and standard deviation of the embedding of the input. 
    Then a two-layer FCN takes a sampled embedding vector as input and outputs the confidence.
    The hidden layer in this FCN has $100$ neurons and a $tanh$ activation function. 
    The output layer uses the $Sigmoid$ function to output the confidence.
    In the CNN version, the actor and critic networks share $3$ convolutional layers, each having $5$, $2$, $2$ filters, $2\times2$ kernel size, and $ReLU$ activation function. 
    The last convolutional layer is concatenated with a two-layer FCN. 
    The hidden layer has $64$ neurons and uses $tanh$ as the activation function.
    The output layer uses the $Sigmoid$ function as the activation function.
    \end{itemize}
    
    \noindent{\bf{Hyperparameters}}
       The hyperparameters that appear in Algorithm \ref{alg:app_b_2} and \ref{alg:app_b_2} are summarized in Table \ref{tab:app_c_1} where we use  {\tt{N/A}} to indicate using $\delta^*$, in which case we let $\mu =0$.  
       Otherwise, the values of $\mu$ and $\delta$ vary depending on the task and IRL algorithm.
       The parameter $\lambda_0$ is the initial value of $\lambda$ as explained in Appendix \ref{subsec:app_b_3}.

\begin{table}[ht]
\centering
\begin{tabular}{r|c|c}
    Parameter & Continuous Control Domain & Partially Observable Domain\\
    Policy training batch size & 64  & 256 \\
    Discount factor  & 0.99 & 0.99 \\
    GAE parameter & 0.95 & 0.95 \\
    PPO clipping parameter & 0.2 & 0.2 \\
    $\lambda_0$ & 1e3 &  1e3 \\
    $\sigma$ & 0.2 & 0.2 \\
    $i_c$ & 0.5 & 0.5 \\
    $\beta$ & 0.0 & 0.0 \\
    $\mu$ & VAIL(HalfCheetah): 0.5; others: 0.0 & VAIL: 1.0; GAIL: 1.0\\
    $\delta$ & VAIL(HalfCheetah): 1.0; others: N/A & VAIL: 0.8; GAIL: 1.2\\
\end{tabular}
\caption{Hyperparameters used in the training processes}
\label{tab:app_c_1}
\end{table}

\textbf{Expert Demonstrations.} Our expert demonstrations all achieve high rewards in the task. 
The number of trajectories and the average trajectory total rewards are listed in Table \ref{tab:app_c_2}.

\begin{table}
\begin{tabular}{r|c|c}
    Task & Number of Trajectories & Average Tot.Rewards\\
     Walker2d-v2 & 10 & 4133 \\
     HalfCheetah-v2 & 100 & 1798 \\
     Hopper-v2 & 100 & 3586 \\
     InvertedPendulum-v2 & 10 & 1000\\
     Swimmer-v2 & 10 & 122 \\
     DoorKey-6x6-v0 &10 & 0.92\\
     SimpleCrossingS9N1-v0 & 10 & 0.93\\ 
\end{tabular}
\caption{The number of demonstrated trajectories and the average trajectory rewards}
\label{tab:app_c_2}
\end{table}

\subsection{Additional Results}\label{subsec:app_c_2}
 
We append the results in three Mujoco benchmarks: \textit{Hopper-v2}, \textit{InvertedPendulum-v2} and \textit{Swimmer-v2} in Figure \ref{fig:app_c_1}. 
Algorithm \ref{alg:pagar2_1} performs similarly to VAIL and GAIL in those two benchmarks. 
IQ-learn does not perform well in Walker2d-v2 but performs better than ours and other baselines by a large margin.

\begin{figure}[htb]
\begin{subfigure}[b]{0.32\linewidth}
        \includegraphics[width=1.1\linewidth]{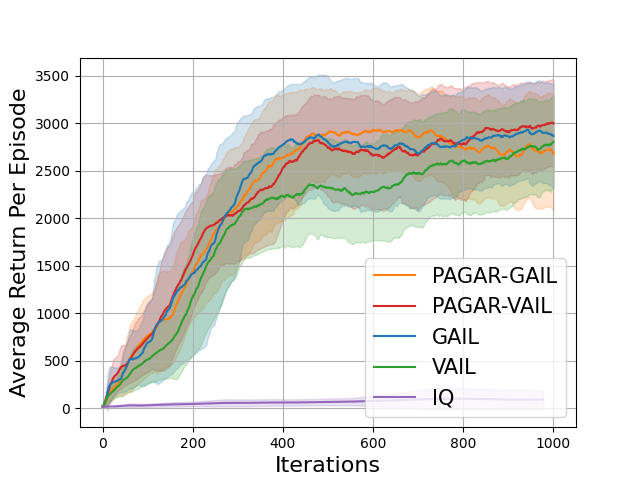}
        \caption{Hopper-v2}
    \end{subfigure} 
\hfill
\begin{subfigure}[b]{0.32\linewidth}
        \includegraphics[width=1.1\linewidth]{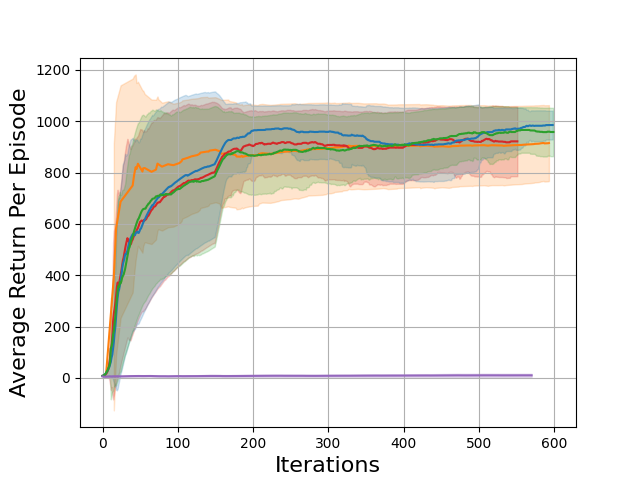}
        \caption{InvertedPendulum-v2}
    \end{subfigure}    
\hfill
\begin{subfigure}[b]{0.32\linewidth}
        \includegraphics[width=1.1\linewidth]{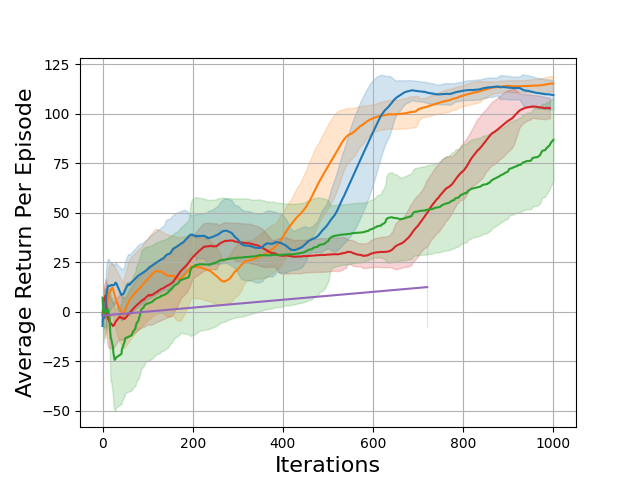}
        \caption{Swimmer-v2}
    \end{subfigure}%
\hfill
\caption{Comparing Algorithm \ref{alg:pagar2_1} with baselines.  The suffix after each `{\tt{PAGAR-}}' indicates which IRL algorithm is utilized in Algorithm 1. The $y$ axis is the average return per step. The $x$ axis is the number of iterations in GAIL, VAIL, and ours. The policy is executed between each iteration for $2048$ timesteps for sample collection.
One exception is that IQ-learn updates the policy at every timestep, making its actual number of iterations $2048$ times larger than indicated in the figures. 
}\label{fig:app_c_1}
\end{figure}

\subsection{Influence of Reward Hypothesis Space}\label{subsec:app_c_3}

In addition to the \textit{DoorKey-6x6-v0} environment, we also tested PAGAR-GAIL and GAIL in \textit{SimpleCrossingS9N2-v0} environment.
The results are shown in Figure \ref{fig:app_c_2}. 

\begin{figure}[htb]
\begin{subfigure}[b]{0.45\linewidth}
        \includegraphics[width=1.1\linewidth]{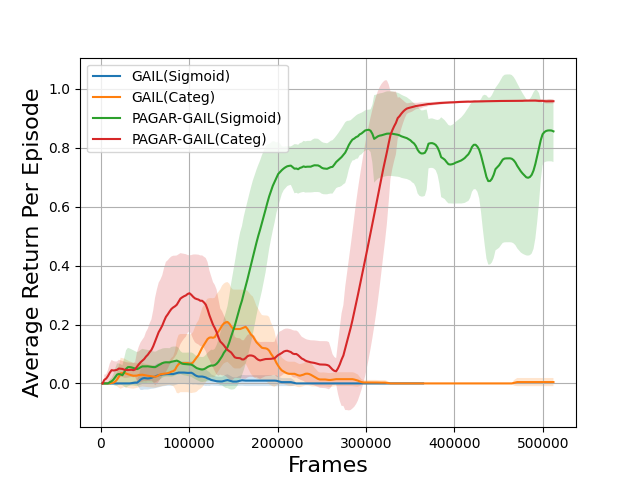}
        \caption{MiniGrid-DoorKey-6x6-v0}
    \end{subfigure}    
\hfill
\begin{subfigure}[b]{0.45\linewidth}
        \includegraphics[width=1.1\linewidth]{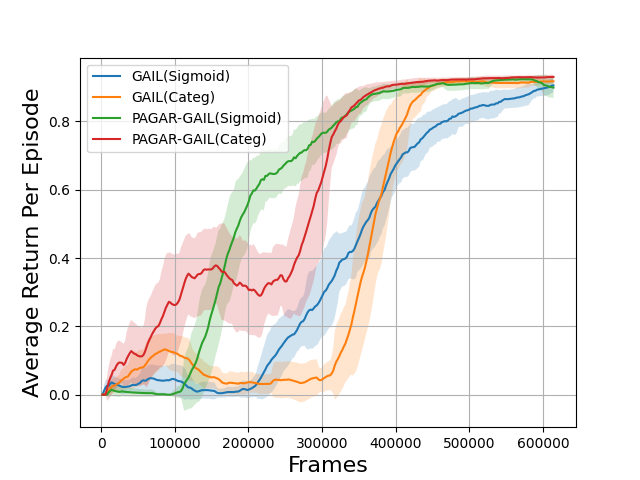}
        \caption{MiniGrid-SimpleCrossingS9N2-v0}
    \end{subfigure}%
\hfill
\caption{Comparing Algorithm \ref{alg:pagar2_1} with baselines.  The prefix `{\tt{protagonist\_GAIL}}' indicates that the IRL algorithm utilized in Algorithm 1 is the same as in GAIL. The `{\tt{\_Sigmoid}}' and `{\tt{\_Categ}}' suffixes indicate whether the output layer of the discriminator is using the $Sigmoid$ function or Categorical distribution. The $x$ axis is the number of sampled frames.  The $y$ axis is the average return per episode. 
}\label{fig:app_c_2}
\end{figure}


\end{document}